\PassOptionsToPackage{table, usenames, dvipsnames}{xcolor}
\documentclass[10pt,journal,compsoc]{IEEEtran}

\usepackage{dirtytalk}

\usepackage[style=numeric-comp, natbib=true, sorting=none, backend=biber, backref=true]{biblatex}
\ifCLASSINFOpdf
\else
\fi

\usepackage{pgfplots}
\usepackage{pgfplotstable}

\usetikzlibrary{decorations.pathreplacing}
\usetikzlibrary{arrows.meta}
\usetikzlibrary{shapes.geometric}

\pgfplotsset{compat=1.16}

\usepackage[subrefformat=parens, labelformat=parens]{subcaption}

\usepackage{amsthm}
\usepackage{amsmath}
\usepackage{mathtools}
\usepackage{amsfonts}
\usepackage{bm}
\usepackage{bbm}

\usepackage{algorithm}
\usepackage{algpseudocode}

\usepackage{enumitem}

\usepackage[normalem]{ulem}

\usepackage{booktabs}
\usepackage{multirow}
\usepackage{tablefootnote}

\usepackage{thmtools, thm-restate}
\declaretheorem{theorem}

\newtheorem{remark}{Remark}
\newtheorem*{definition}{Definition}
\newtheorem{lemma}{Lemma}
\newtheorem{corollary}{Corollary}

\newenvironment{proofsketch} {\begin{proof}[Proof sketch]} {\end{proof}}

\newcommand\MYhyperrefoptions{bookmarks=true, bookmarksnumbered=true, pdfpagemode={UseOutlines}, plainpages=false, pdfpagelabels=true, colorlinks=true, linkcolor={BrickRed}, citecolor={MidnightBlue}, urlcolor={PineGreen}, pdftitle={PAC-Bayes meta-learning}, 
pdfsubject={meta-learning},
pdfauthor={Cuong Nguyen},
pdfkeywords={meta-learning, PAC-Bayes}}
\ifCLASSINFOpdf
\usepackage[\MYhyperrefoptions, pdftex, pagebackref=false]{hyperref}
\else
\usepackage[\MYhyperrefoptions, breaklinks=true, dvips, pagebackref=false]{hyperref}
\usepackage{breakurl}
\fi

\renewcommand{\appendixname}{Supplemental Material}

\newif\ifhighlight

\highlightfalse

\ifhighlight
    \newcommand{\red}[1]{\textcolor{BrickRed}{#1}}
    \newcommand{\green}[1]{\textcolor{PineGreen}{#1}}
    \newcommand{\orange}[1]{\textcolor{BurntOrange}{#1}}
    \newcommand{\blue}[1]{\textcolor{MidnightBlue}{#1}}
    \newcommand{\purple}[1]{\textcolor{RoyalPurple}{#1}}
    \newcommand{\magenta}[1]{\textcolor{Magenta}{#1}}
\else
    \newcommand{\red}[1]{#1}
    \newcommand{\green}[1]{#1}
    \newcommand{\orange}[1]{#1}
    \newcommand{\blue}[1]{#1}
    \newcommand{\purple}[1]{#1}
    \newcommand{\magenta}[1]{#1}
\fi

\begin{document}
%
\def\mytitle{PAC-Bayes Meta-learning with Implicit Task-specific Posteriors}
\title{\mytitle}
%
%
%
%

\author{Cuong~Nguyen\href{https://orcid.org/0000-0003-2672-6291}{\includegraphics[width=1em]{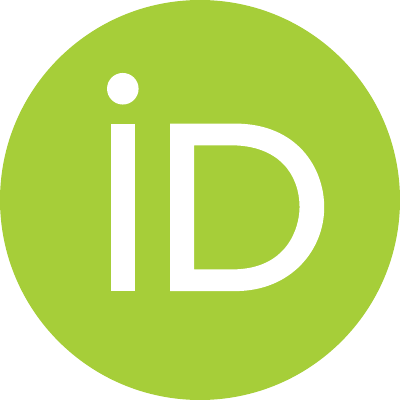}},
        Thanh-Toan~Do,
        and~Gustavo~Carneiro\href{https://orcid.org/0000-0002-5571-6220}{\includegraphics[width=1em]{img/orcid_icon.pdf}}
\IEEEcompsocitemizethanks{\IEEEcompsocthanksitem C. Nguyen and G. Carneiro are with the Australian Institute for Machine Learning, University of Adelaide, SA, Australia 5000.\protect\\
\IEEEcompsocthanksitem T.-T. Do is with the Department of Data Science and AI, Faculty of Information Technology, Monash University.}
\thanks{Work in progress.}}

%
%

\markboth{}
{Nguyen \MakeLowercase{\textit{et al.}}: \mytitle}
%



\IEEEtitleabstractindextext{%
\begin{abstract}
    We introduce a new and rigorously-formulated PAC-Bayes meta-learning algorithm that solves few-shot learning. Our proposed method extends the PAC-Bayes framework from a single task setting to the meta-learning multiple task setting to upper-bound the error evaluated on any, even unseen, tasks and samples. We also propose a generative-based approach to estimate the posterior of task-specific model parameters more expressively compared to the usual assumption based on a multivariate normal distribution with a diagonal covariance matrix. We show that the models trained with our proposed meta-learning algorithm are well calibrated and accurate, with state-of-the-art calibration and classification results on few-shot classification (mini-ImageNet and tiered-ImageNet) and regression (multi-modal task-distribution regression) benchmarks.
\end{abstract}

}

\maketitle

\IEEEdisplaynontitleabstractindextext

%
\IEEEpeerreviewmaketitle

\ifCLASSOPTIONcompsoc
\IEEEraisesectionheading{\section{Introduction}\label{sec:introduction}}
\else
\section{Introduction}
\label{sec:introduction}
\fi

\IEEEPARstart{O}{ne} unique ability of humans is to quickly learn new tasks with only a few \textit{training} examples. This is due to the fact that humans tend to exploit prior experience to facilitate the learning of new tasks. Such exploitation is markedly different from  conventional machine learning approaches, where no prior knowledge (e.g. training from scratch with random initialisation)~\cite{glorot2010understanding}, or weak prior knowledge (e.g., fine tuning from pre-trained models)~\cite{rosenstein2005transfer} are employed to learn a new task. This motivates the development of novel learning algorithms that can effectively encode the knowledge learnt from training tasks, and exploit that knowledge to quickly adapt to future tasks~\cite{lake2015human}.

Prior knowledge can be helpful for future learning only if all tasks are assumed to be distributed according to a latent task distribution. Learning this latent distribution is, therefore, useful for solving an unseen task, even if the task contains a limited number of training examples. Many approaches have been proposed and developed to achieve this goal, namely: \textit{multi-task learning}~\cite{caruana1997multitask}, \textit{domain adaptation}~\cite{bridle1991recnorm,ben2010theory} and \textit{meta-learning}~\cite{schmidhuber1987evolutionary,thrun1998learning}. Among these, meta-learning has flourished as one of the most effective methods due to its ability to leverage the knowledge learnt from many training tasks to quickly adapt to unseen tasks.

Recent advances in meta-learning have produced state-of-the-art results in many benchmarks of few-shot learning data sets~\cite{santoro2016meta,ravi2017optimization,munkhdalai2017meta,snell2017prototypical,finn2017model,zhang2018metagan,rusu2019meta}. Learning from a few training examples is often difficult and easily leads to over-fitting, especially when no model uncertainty is taken into account. This issue has been addressed by several recent probabilistic meta-learning approaches that incorporate model uncertainty into prediction, e.g., LLAMA (based on Laplace method)~\cite{grant2018recasting}, or PLATIPUS~\cite{finn2017model}, Amortised Bayesian Meta-learner (ABML)~\cite{ravi2018amortized} and VERSA~\cite{gordon2018metalearning} that use variational inference (VI). However, these studies have not thoroughly investigated the errors evaluated on arbitrary tasks (including seen and unseen) sampled from the same task distribution and arbitrary samples generated from the same task. This results in limited theoretical generalisation guarantees. Moreover, most of these studies are based on variational functions that may not represent well the richness of the underlying distributions. For instance, a common choice for the variational distribution relies on a multivariate normal distribution with a diagonal covariance matrix, which can potentially worsen the prediction accuracy given its limited representability.

In this paper, we address the two problems listed above with the following technical novelties: (i) derivation of a rigorous meta-learning objective that upper-bounds the errors evaluated on any tasks and any samples of few-shot learning setting based on the PAC-Bayes framework, and (ii) proposal of a novel implicit modelling approach to expressively represent the posterior of task-specific model parameter. Our evaluation shows that the models trained with our proposed meta-learning algorithm is at the same time well calibrated and accurate, with state-of-the-art results in few-shot classification (mini-ImageNet and tiered-ImageNet) and regression (multi-modal task-distribution regression) benchmarks in terms of accuracy, Expected Calibration Error (ECE) and Maximum Calibration Error (MCE).
\section{Related Work}
\label{sec:related_work}
    Our paper is related to probabilistic few-shot meta-learning techniques that have been developed to incorporate uncertainty into model estimation. LLAMA~\cite{grant2018recasting} employs the Laplace method to extend the deterministic estimation assumed in MAML~\citep{finn2017model} to a multivariate normal distribution. However, the need to estimate and invert the Hessian matrix of a loss function makes this approach computationally challenging for large-scale models, such as deep neural networks. Variational inference (VI) addresses such scalability issue -- remarkable examples of VI-based methods are PLATIPUS~\cite{finn2018probabilistic}, BMAML~\cite{yoon2018bayesian}, ABML~\cite{ravi2018amortized} and VERSA~\cite{gordon2018metalearning}. Although these VI-based approaches have demonstrated impressive results in regression, classification as well as reinforcement learning, they do not provide any theoretical guarantee on the error induced by arbitrary or even unseen tasks sampled from the same task distribution as well as any samples belonging to the same task. Moreover, the variational distributions used in most of these works are overly-simplified as multivariate normal distribution with diagonal covariance matrices. This assumption, however, limits the expressiveness of the variational approximation, resulting in a less accurate prediction.

    \green{Our work is also related to the PAC-Bayes framework used in meta-learning that upper-bounds errors with certain confidence levels~\cite{pentina2014pac, amit18meta}. The main difference between these previous works and ours is at the modelling of meta-parameter and its objective function. In the previous works, the meta-parameter is the prior of task-specific parameter which is analogous to a regularisation in task adaptation step, while in our proposed method, the meta-parameter is the model initialisation. In addition, the existing works rely on a train-train setting~\citep{bai2021important} where the all data of a task is used for task adaptation such as REPTILE~\citep{nichol2018first}, while ours follows the train-validation split with the bi-level optimisation objective shown below in \eqref{eq:meta_learning_objective}. Such differences lead to a discrepancy in the formulation of the corresponding PAC-Bayes bounds. Another work closely related to our proposed method is \emph{exponentially weighted aggregation for lifelong learning} (EWA-LL)~\citep{alquier2017regret}. In EWA-LL, each task-specific model is decomposed into a shared feature extractor and a task-specific classifier, while in our approach, each model is an adapted or a fine-tuned version of the meta-parameter. Moreover, the setting of EWA-LL follows the train-train meta-learning approach, making the algorithm analogous to multi-task learning, while our proposed method is a train-validation meta-learning approach with the bi-level optimisation objective.}
    

    \green{Our work has a connection to the statistical analysis of meta-learning that proves generalisation upper-bound for meta-learning algorithms~\citep{maurer2005algorithmic, maurer2016benefit}. Some typical recent works include the learning of the common regularisation that is used when adapting or fine-tuning on a specific task~\citep{denevi2018learning, denevi2019learning, denevi2019online, denevi2020advantage} to improve the performance of meta-learning algorithms in heterogeneous task environments, or analyse and optimise the regret induced by meta-learning algorithms in an online setting~\citep{khodak2019adaptive}. Our work differs from this line of works at how the meta-parameter is modelled. In our case, the meta-parameter of interest is the model initialisation, and our goal is to learn a variational distribution for such parameter, while existing works consider different parameter, such as the shared L2 regularisation parameters, as meta-parameter, and often learn a point estimate for such meta-parameters.}

\section{Background}
\label{sec:background}
    \subsection{Data generation model of a task}
    \label{sec:data_generation}
        \blue{A data point of a task indexed by \(i\) considered in this paper consists of an input \(\mathbf{x}_{ij} \in \mathcal{X} \subseteq \mathbb{R}^{d}\) and a corresponding label \(\mathbf{y}_{ij} \in \mathcal{Y}\) with \(j \in \mathbb{N}\). Such data points are generated in 2 steps. The first step is to generate the input \(\mathbf{x}_{ij}\) by sampling from some probability distribution \(\mathcal{D}_{i}\). The second step is to determine the label \(\mathbf{y}_{ij} = f(\mathbf{x}_{ij})\), where \(f_{i}: \mathcal{X} \to \mathcal{Y}\) is the \say{correct} labelling function. Note that both the probability distribution \(\mathcal{D}_{i}\) and the labelling function \(f_{i}\) are unknown. To simplify the notations, \((\mathbf{x}_{ij}, \mathbf{y}_{ij}) \sim (\mathcal{D}_{i}, f_{i})\) is then used to denote such data generation.}

    \subsection{Task instance}
    \label{sec:task}
        \begin{definition}{\citep{hospedales2021meta}}
            \label{def:task_instance}
            \blue{A task or a task instance \(\mathcal{T}_{i}\) consists of an unknown associated data generation model \((\mathcal{D}_{i}, f_{i})\), and a loss function \(\ell_{i}\), denoted as: \(\mathcal{T}_{i} = \{(\mathcal{D}_{i}, f_{i}), \ell_{i}\}\).}
        \end{definition}

        \begin{remark}
            \blue{The loss function \(\ell_{i}\) is defined abstractly, and not necessarily some common loss functions, such as mean squared error (MSE) or cross-entropy. For example, \(\ell_{i}\) could be referred to as negative log-likelihood if the objective is maximum likelihood estimation, or variational-free energy if the objective is based on variational inference.}
        \end{remark}

        \blue{To solve a task \(\mathcal{T}_{i}\), one needs to obtain an optimal task-specific model \(h(.; \mathbf{w}_{i}^{*}): \mathcal{X} \to \mathcal{Y}\), parameterised by \({\mathbf{w}^{*}_{i} \in \mathcal{W} \subseteq \mathbb{R}^{n}}\), which minimises a loss function \(\ell\) on the data of that task:}
        \begin{equation}
            \mathbf{w}_{i}^{*} = \arg\min_{\mathbf{w}_{i}} \mathbb{E}_{(\mathbf{x}_{ij}, \mathbf{y}_{ij}) \sim (\mathcal{D}_{i}, f_{i})} \left[ \ell_{i} (\mathbf{x}_{ij}, \mathbf{y}_{ij}; \mathbf{w}_{i}) \right].
        \end{equation}
        
        \blue{In practice, since both \(\mathcal{D}_{i}\) and \(f_{i}\) are unknown, the data generation model is replaced by a dataset consisting of a finite number of data-points generated according to the data generation model \((\mathcal{D}_{i}, f_{i})\), denoted as \(S_{i} = \{\mathbf{x}_{ij}, \mathbf{y}_{ij}\}_{j=1}^{m_{i}}\). The objective to solve that task is often known as empirical risk minimisation (ERM):}
        \begin{equation}
            \mathbf{w}^{\mathrm{ERM}}_{i} = \arg\min_{\mathbf{w}_{i}} \frac{1}{m_{i}} \sum_{j = 1}^{m_{i}} \left[ \ell_{i} (\mathbf{x}_{ij}, \mathbf{y}_{ij}; \mathbf{w}_{i}) \right].
            \label{eq:objective_minimise_loss}
        \end{equation}

        \blue{For simplicity, this paper considers two families of tasks: regression and classification. As a result, the label is a scalar \(\mathcal{Y} \subseteq \mathbb{R}\) for regression and \(\mathcal{Y} = \{0, 1, \ldots, n - 1\}\) for classification. In addition, the loss function used will be the same for each task family, hence, the subscript on the loss function will be dropped, and the loss is denoted as \(\ell\) throughout the paper. Due to the commonality of the loss function across all tasks, we will drop the notation of \(\ell\) when referring to a task. In other words, a task can be simply represented by either its data generation model \((\mathcal{D}_{i}, f_{i})\) or the associated dataset \(S_{i}\).}

    \subsection{Meta-learning}
    \label{sec:meta_learning}
        \begin{figure}[t]
            \centering
            \begin{tikzpicture}[scale=1, every node/.style={scale=1}]
                \pgfmathsetmacro{\yshift}{0.75}
                \pgfmathsetmacro{\xshift}{2.}
                \pgfmathsetmacro{\minsize}{0.9}
                \node[shape=circle, draw=black, minimum size=\minsize cm] at (0, 0) (theta) {\(\theta\)};
                \node[shape=circle, draw=black, minimum size=\minsize cm] at ([xshift=\xshift cm]theta) (w) {\(\lambda\)};
                
                \node[shape=circle, draw=black] at ([xshift=\xshift cm, yshift=\yshift cm]w) (y_t) {\(y^{(t)}\)};
                \node[shape=circle, draw=black] at ([xshift=\xshift cm]y_t) (x_t) {\(\mathbf{x}^{(t)}\)};
                
                \node[shape=circle, draw=black] at ([xshift=\xshift cm, yshift=-\yshift cm]w) (y_v) {\(y^{(v)}\)};
                \node[shape=circle, draw=black] at ([xshift=\xshift cm]y_v) (x_v) {\(\mathbf{x}^{(v)}\)};
                
                \draw[-Latex] (theta) to (w);
                \draw[-Latex] (w) to (y_t);
                \draw[-Latex] (w) to (y_v);
                \draw[-Latex] (x_t) to (y_t);
                \draw[-Latex] (x_v) to (y_v);
                
                \draw[-Latex, dashed] (y_t.north) to [out=150, in=60] (w.north);
                \draw[-Latex, dashed] (w.south) to [out=-150, in=-30] (theta.south);
                \draw[-Latex, dashed] (y_v.south) to [out=-150, in=-45] (theta.south);
                
                \draw[rounded corners] ([xshift=-0.5*\xshift cm, yshift=-0.8*\yshift cm]y_t) rectangle ([xshift=0.4*\xshift cm, yshift=0.8*\yshift cm]x_t);
                \draw[rounded corners] ([xshift=-0.5*\xshift cm, yshift=-0.8*\yshift cm]y_v) rectangle ([xshift=0.4*\xshift cm, yshift=0.8*\yshift cm]x_v);
                \draw[rounded corners] ([xshift=-0.5*\xshift cm, yshift=-2*\yshift cm]w) rectangle ([xshift=0.5*\xshift cm, yshift=\yshift cm]x_t);
            \end{tikzpicture}
            \caption{Meta-learning is an extension of hyper-parameter optimisation, where the meta-parameter \(\theta\) is shared across all tasks. The solid arrows denote forward pass, while the dashed arrows indicate parameter inference, and rectangles illustrate the plate notations. The training subset \(\{ ( \mathbf{x}_{ij}^{(t)}, y_{ij}^{(t)} ) \}_{j=1}^{m_{i}^{(t)}}\) of task \(\mathcal{T}_{i}\) and the meta-parameter \(\theta\) are used to learn the task-specific parameter \(\lambda_{i}\), corresponding to the lower-level optimisation in \eqref{eq:meta_learning_objective}. The obtained \(\lambda_{i}\) is then used to evaluate the error on the validation subset \(\{( \mathbf{x}_{ij}^{(v)}, y_{ij}^{(v)} )\}_{j=1}^{m_{i}^{(v)}}\) to learn the meta-parameter \(\theta\), corresponding to the upper-level optimisation in \eqref{eq:meta_learning_objective}.}
    	    \label{fig:meta_learning_diagram}
        \end{figure}
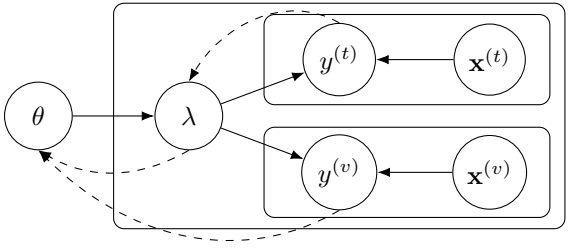

        \blue{The setting of the meta-learning problem considered in this paper follows the \emph{task environment}~\citep{baxter2000model} that describes the unknown distribution \(p(\mathcal{D}, f)\) over a family of tasks. Each task \(\mathcal{T}_{i}\) is sampled from this task environment and can be represented as \(\left( \mathcal{D}_{i}^{(t)}, \mathcal{D}_{i}^{(v)}, f \right)\), where \(\mathcal{D}_{i}^{(t)}\) and \(\mathcal{D}_{i}^{(v)}\) are the probability of training and validation input data, respectively, and they are not necessarily identical. The aim of meta-learning is to obtain a model trained on available training tasks such that the model can be fine-tuned on some labelled data of a testing task drawn from the same task environment to predict the label of unlabelled data on the same task accurately.}
        
        \blue{Such meta-learning methods use meta-parameters to model the common latent structure of the task distribution \(p(\mathcal{D}, f)\). Examples of meta-parameters are: model initialisation~\citep{finn2017model, finn2018probabilistic, ravi2018amortized, nguyen2020uncertainty}, learning rate when fine-tuning the model for a task~\citep{li2017meta}, feature extractor~\citep{vinyals2016matching, snell2017prototypical}, and optimiser~\citep{andrychowicz2016learning, ravi2017optimization}. In this point of view, meta-learning can be considered as an extension of hyper-parameter optimisation in single-task learning, where the hyper-parameter of interest or meta-parameter is shared across many tasks. Mathematically, the objective of meta-learning can be written as a bi-level optimisation:}
        \begin{equation}
            \begin{aligned}[b]
                & \min_{\psi} \mathbb{E}_{q(\theta; \psi)} \mathbb{E}_{p(\mathcal{D}, f)} \mathbb{E}_{q(\mathbf{w}_{i}; \lambda_{i}^{*}(\theta))} \mathbb{E}_{\left( \mathcal{D}_{i}^{(v)}, f_{i} \right)} \left[ \ell \left(  \mathbf{x}_{ij}^{(v)}, y_{ij}^{(v)}; \mathbf{w}_{i} \right) \right]\\
                & \text{s.t.: } \lambda_{i}^{*}(\theta) = \arg\min_{\lambda_{i}(\theta)} \mathbb{E}_{\left( \mathcal{D}_{i}^{(t)}, f_{i} \right)} \mathbb{E}_{q(\mathbf{w}_{i}; \lambda_{i}(\theta))} \left[ \ell \left( \mathbf{x}_{ij}^{(t)}, y_{ij}^{(t)}; \mathbf{w}_{i} \right) \right],\\
                & i \in \mathbb{N},
            \end{aligned}
            \label{eq:meta_learning_objective}
        \end{equation}
        \blue{where \(q(\theta; \psi)\), parameterised by \(\psi\), is a distribution over the meta-parameter \(\theta\), \(q(\mathbf{w}_{i}; \lambda(\theta))\), parameterised by \(\lambda_{i}(\theta)\), is a distribution of task-specific parameter \(\mathbf{w}_{i}\), \(\mathbb{E}_{ \left( \mathcal{D}_{i}^{(t)}, f_{i} \right)}\) means the expectation of input \(\mathbf{x}_{ij}^{(t)}\) sampled from \(\mathcal{D}_{i}^{(t)}\) and its label \(y_{ij}^{(t)} = f \left( \mathbf{x}_{ij}^{(t)} \right)\), and \(\mathbb{E}_{\left( \mathcal{D}_{i}^{(v)}, f_{i} \right)}\) is defined similarly. In addition, to simplify the notations, we drop the dependence of \(\theta\) from \(\lambda_{i}(\theta)\).}
        
        \blue{Note that the difference of the objective in \eqref{eq:meta_learning_objective} from hyper-parameter optimisation in single-task learning is at the upper-level where \eqref{eq:meta_learning_objective} consists of the additional expectation over all training tasks, denoted as \(\mathbb{E}_{p \left(\mathcal{D}, f \right)}\).}
        
        \blue{Depending on how the loss function \(\ell\) and distributions \(q(\theta; \psi)\) and \(q(\mathbf{w}_{i}; \lambda_{i})\) are defined, one can obtain different meta-learning algorithms. For example, if \(\ell\) corresponds to the loss in maximum likelihood and the lower-level is optimised by gradient descent with \(\theta\) as the initialisation of \(\lambda_{i}\):}
        \begin{equation}
            \begin{dcases}
                \ell(\mathbf{x}, y; \mathbf{w}) & = -\ln p(y | \mathbf{x}; \mathbf{w})\\
                \theta & = \text{initialisation of } \lambda_{i}\\
                q(\mathbf{w}_{i}; \lambda_{i}) & = \delta ( \mathbf{w}_{i} - \lambda_{i} )\\
                q(\theta; \psi) & = \delta \left( \theta - \psi \right),
            \end{dcases}
            \label{eq:maml_assumption}
        \end{equation}
        \blue{where \(\delta(.)\) is the Dirac delta function, then the objective in \eqref{eq:meta_learning_objective} can be simplified to:}
        \begin{equation}
            \begin{aligned}[b]
                & \min_{\theta} \mathbb{E}_{p(\mathcal{D}, f)} \mathbb{E}_{\left( \mathcal{D}_{i}^{(v)}, f_{i} \right)} \left[ -\ln p \left( y_{ij}^{(v)} | \mathbf{x}_{ij}^{(v)}; \mathbf{w}_{i} \right) \right]\\
                & \text{s.t.: } \mathbf{w}_{i}^{*} = \arg\min_{\mathbf{w}_{i}} \mathbb{E}_{\left( \mathcal{D}_{i}^{(t)}, f_{i} \right)} \left[ -\ln p \left( y_{ij}^{(t)} | \mathbf{x}_{ij}^{(t)}; \mathbf{w}_{i} \right) \right], i \in \mathbb{N},
            \end{aligned}
            \label{eq:maml_objective}
        \end{equation}
        \blue{which resembles the MAML algorithm~\citep{finn2017model}.}

        \blue{Another example is the probabilistic meta-learning algorithm that replaces the loss function \(\ell\) by a form of the variational-free energy and uses some similar assumptions in MAML:}
        \begin{equation}
            \begin{dcases}
                \ell(\mathbf{x}, y; \mathbf{w}) & = -\ln p(y | \mathbf{x}; \mathbf{w}) + \ln \left[ \frac{q(\mathbf{w}_{i}; \lambda_{i})}{p(\mathbf{w}_{i})} \right]\\
                \theta & = \text{initialisation of } \lambda_{i}\\
                q(\mathbf{w}_{i}; \lambda_{i}) & = \mathcal{N} \left( \mathbf{w}_{i}; \bm{\mu}_{\lambda_{i}}, \mathrm{diag} \left( \bm{\sigma}_{\lambda_{i}}^{2} \right) \right)\\
                q(\theta; \psi) & = \delta \left( \theta - \psi \right),
            \end{dcases}
        \end{equation}
        \blue{where \(\mathcal{N}(.)\) denotes multivariate normal distribution, \(\mathrm{diag}(.)\) denotes a diagonal matrix, and \(p(\mathbf{w}_{i})\) is the prior of \(\mathbf{w}_{i}\). This formulation resembles ABML~\citep{ravi2018amortized} and VAMPIRE~\citep{nguyen2020uncertainty} algorithms.}

    \subsection{PAC-Bayes upper-bound in single-task learning}
    \label{sec:pac_bayes_single_task}
        \blue{In practice, the data probability distribution \(\mathcal{D}_{i}\) and the labelling function \(f_{i}\) of a task \(\mathcal{T}_{i}\) are unknown, but only a dataset \(S_{i}\) consisting of finite input data and labels is available. Since the aim is to minimise the loss averaged over all data generated from \((\mathcal{D}_{i}, f_{i})\), it is, therefore, important to analyse the difference between such \say{true} loss and the empirical loss evaluated on a given dataset. Such difference can be upper-bounded by the Kullback-Leibler (KL) divergence shown in Theorem~\ref{theorem:pac_bayes_single_task} with a certain level of confidence.}

        \begin{theorem}{\citep{mcallester1999pac}}
        \label{theorem:pac_bayes_single_task}
            \blue{If a dataset \(S_{i}\) consists of \(m_{i}\) inputs \(\mathbf{x}_{ik}\) i.i.d. sampled from a data probability distribution \(\mathcal{D}_{i}\) and being labelled by \(f_{i}\), \(\mathcal{H}\) is a hypothesis class, \(\varepsilon \in (0, 1]\) and a loss function \(\ell: \mathcal{H} \times \mathcal{Y} \to [0, 1]\), then for any \say{posterior} \(q(\mathbf{w}_{i}; \lambda_{i})\) over a hypothesis \(h(.; \mathbf{w}_{i}) \in \mathcal{H}\), parameterised by \(\mathbf{w}_{i}\), the following holds with a probability at least \(1 - \varepsilon\):}
            \begin{equation*}
                \begin{aligned}[b]
                    & \mathbb{E}_{q(\mathbf{w}_{i}; \lambda_{i})} \mathbb{E}_{(\mathcal{D}_{i}, f_{i})} \left[ \ell (\mathbf{x}_{ij}, y_{ij}; \mathbf{w}_{i}) \right] \\
                    & \le \frac{1}{m_{i}} \sum_{k = 1}^{m_{i}} \mathbb{E}_{q(\mathbf{w}_{i}; \lambda_{i})} \left[ \ell(\mathbf{x}_{ik}, y_{ik}; \mathbf{w}_{i}) \right] \\
                    & \quad + \sqrt{ \frac{\mathrm{KL} \left[ q(\mathbf{w}_{i}; \lambda_{i}) || p(\mathbf{w}_{i}) \right] + \frac{\ln m_{i}}{\varepsilon}}{2(m_{i} - 1)}},
                \end{aligned}
            \end{equation*}
            \blue{where \(p(\mathbf{w}_{i})\) is the prior of \(\mathbf{w}_{i}\).}
        \end{theorem}

        \blue{Instead of minimising the \say{true} loss of a task, denoted as the left-hand side term in Theorem~\ref{theorem:pac_bayes_single_task}, which is intractable, one should minimise both the empirical loss and the KL divergence on the right-hand side. Indeed, the upper-bound in Theorem~\ref{theorem:pac_bayes_single_task} is often used as a tractable learning objective function for the model of interest.}
\section{Methodology}
\label{sec:methodology}
    \subsection{PAC-Bayes meta-learning}
        \purple{In practice, the training and validation data probability distribution, \(\mathcal{D}_{i}^{(t)}\) and \(\mathcal{D}_{i}^{(v)}\), in meta-learning, are unknown, and only two datasets with finite examples, \(S_{i}^{(t)} = \left\{ \left( \mathbf{x}_{ij}^{(t)}, y_{ij}^{(t)} \right) \right\}_{j=1}^{m_{i}^{(t)}}\) and \(S_{i}^{(v)} = \left\{ \left( \mathbf{x}_{ik}^{(v)}, y_{ik}^{(v)} \right) \right\}_{k=1}^{m_{i}^{(v)}}\) associated with task \(\mathcal{T}_{i}\), are provided as illustrated in \figureautorefname~\ref{fig:meta_learning_diagram}. Note that \(m_{i}^{(t)}\) and \(m_{i}^{(v)}\) are not necessarily identical, and the label \(y_{ij}^{(v)}\) in the validation dataset \(S_{i}^{(v)}\) are known in training, while being unknown in testing. In addition, there are only \(T\) tasks sampled from the task environment that are available for training. Hence, it is important to derive an upper-bound, and in particular PAC-Bayes upper-bound, for both the \say{generalisation} losses in upper- and lower-levels of \eqref{eq:meta_learning_objective}. Since the lower-bound corresponds to solving a single task, the upper-bound derived in Theorem~\ref{theorem:pac_bayes_single_task} can be straight-forwardly applied as the learning objective. The remaining problem lies on the formulation of the PAC-Bayes upper-bound for the loss in the upper-level of \eqref{eq:meta_learning_objective}. This novel bound is shown in Theorem~\ref{theorem:pac_bayes_meta_learning} with its detailed proof in \appendixautorefname~\ref{apdx:pac_bayes_proof}.}
    
        \begin{restatable}[]{theorem}{MetaBound}
    	\label{theorem:pac_bayes_meta_learning}
        	\purple{Given \(T\) tasks sampled from the same task environment \(p(\mathcal{D}, f)\), where each task has an associated dataset \(S_{i}\) with samples generated from the task-specific data generation model \((\mathcal{D}_{i}, f_{i})\), then for a bounded loss function \(\ell: \mathcal{W} \times \mathcal{Y} \to [0, 1]\) and any distributions \(q(\theta; \psi)\) of meta-parameter \(\theta\) and \(q(\mathbf{w}_{i}; \lambda_{i})\) of task-specific parameter \(\mathbf{w}_{i}\), the following holds with the probability at least \(1 - \varepsilon, \forall \varepsilon \in (0, 1]\):}
        	\begin{equation*}
        		\begin{aligned}[b]
            		& \mathbb{E}_{q(\theta; \psi)} \mathbb{E}_{p(\mathcal{D}, f)} \mathbb{E}_{q(\mathbf{w}_{i}; \lambda_{i})} \mathbb{E}_{(\mathcal{D}_{i}^{(v)}, f_{i})} \left[ \ell \left( \mathbf{x}_{ij}^{(v)}, y_{ij}^{(v)}; \mathbf{w}_{i} \right) \right] \\
            		& \quad \le \frac{1}{T} \sum_{i = 1}^{T} \frac{1}{m_{i}^{(v)}} \sum_{k = 1}^{m_{i}^{(v)}} \mathbb{E}_{q(\theta; \psi)} \mathbb{E}_{q(\mathbf{w}_{i}; \lambda)} \left[ \ell \left( \mathbf{x}_{ik}^{(v)}, y_{ik}^{(v)}; \mathbf{w}_{i} \right) \right] \\
            		& \qquad + \sqrt{ \frac{ \mathbb{E}_{q(\theta; \psi)} \left[ \mathrm{KL} \left[ q(\mathbf{w}_{i}; \lambda_{i}) || p(\mathbf{w}_{i}) \right] \right] + \frac{T^{2}}{(T - 1) \varepsilon}\ln m_{i}^{(v)} }{2 \left( m_{i}^{(v)} - 1 \right)} } \\
            		& \qquad \quad + \sqrt{ \frac{ \mathrm{KL}  \left[ q(\theta; \psi) || p(\theta) \right] + \frac{T \ln T}{\varepsilon} }{2 (T - 1)} },
        		\end{aligned}
        	\end{equation*}
        	\purple{where \(p(\mathbf{w}_{i}), \forall i \in \{1, \ldots, T\}\) is the prior of task-specific parameter \(\mathbf{w}_{i}\) and \(p(\theta)\) is the prior of meta-parameter \(\theta\).}
        \end{restatable}
        
        \begin{proofsketch}
            \purple{The proof is carried out in three steps: (i) derive a PAC-Bayes upper-bound for unseen samples generated from task-specific data generation model \((\mathcal{D}_{i}^{(v)}, f_{i}), \forall i \in \{1, \ldots, T\}\) by adapting the proof of Theorem~\ref{theorem:pac_bayes_single_task} as shown in \appendixautorefname~\ref{sec:pac_bayes_bound_unseen_samples_single_task}, (ii) derive a PAC-Bayes upper-bound for unseen tasks by applying Theorem~\ref{theorem:pac_bayes_single_task} as shown in \appendixautorefname~\ref{sec:pac_bayes_bound_unseen_tasks}, and (iii) combine the two obtained results as shown in \appendixautorefname~\ref{sec:pac_bayes_meta_learning} to complete the proof.}
        \end{proofsketch}

        \begin{remark}
        \label{rm:bounded_loss}
            \purple{One limitation of Theorems~\ref{theorem:pac_bayes_single_task} and \ref{theorem:pac_bayes_meta_learning} is the assumption of bounded loss which restricts the loss function within \([0, 1]\). Although there are several works that extend PAC-Bayes bound for unbounded losses~\citep{catoni2004statistical,germain2016pac,alquier2016properties,alquier2018simpler}, their formulations still needs to assume the boundedness of the moment generating function of some particular loss functions. This assumption is, however, impractical since some common loss functions in practice, such as mean squared error or cross-entropy, do not possess such property. Nevertheless, our main focus in this paper is to provide a theoretical generalisation guarantee for meta-learning using PAC-Bayes theory. Such extension is not necessary since in the implementation our loss is clipped to be within \([0,1]\).}
        \end{remark}

    \subsection{Practical meta-learning objective}
        \purple{Instead of following the objective in \eqref{eq:meta_learning_objective}, which is intractable due to the unknown data generation model of each task and the unknown task environment, we utilise the results in Theorems~\ref{theorem:pac_bayes_single_task} and \ref{theorem:pac_bayes_meta_learning} to propose a new objective function for a practical meta-learning that theoretically guarantees the generalisation errors due to unseen samples of a task and unseen tasks sampled from the task environment. The new objective is formulated by replacing the optimisation in the lower-level by the minimisation of the corresponding upper-bound in Theorem~\ref{theorem:pac_bayes_single_task}, and the optimisation in the upper-level by the minimisation of the upper-bound in Theorem~\ref{theorem:pac_bayes_meta_learning}.}

    \subsection{Meta-learning with implicit task-specific posterior}
    \label{sec:implicit_distribution}
        Given the new objective function where both the lower- and upper-levels are replaced by the minimisation of the PAC-Bayes upper-bounds in Theorems~\ref{theorem:pac_bayes_single_task} and \ref{theorem:pac_bayes_meta_learning}, there are a total of four distributions of interest: the \say{posteriors} \(q(\theta; \psi)\) and \(q(\mathbf{w}_{i}; \lambda_{i}), i \in \{1, \ldots, T\}\), and the priors \(p(\theta)\) and \(p(\mathbf{w}_{i})\).

        As priors represent the modelling assumption and are chosen before observing data, both \(p(\theta)\) and \(p(\mathbf{w}_{i})\) are, therefore, assumed to be multivariate normal distributions with diagonal covariance matrices:
        \begin{align}
            p(\theta) & = \mathcal{N} \left( \theta; \pmb{0}, \sigma_{\theta}^{2} \mathbf{I} \right) \label{eq:prior_theta}\\
            p(\mathbf{w}_{i}) & = \mathcal{N} \left( \mathbf{w}_{i}; \pmb{0}, \sigma_{w}^{2} \mathbf{I} \right), \label{eq:prior_w}
        \end{align}
        where \(\pmb{0}\) is a vector containing all zeros, \(\mathbf{I}\) is the identity matrix, \(\sigma_{\theta} > 0\) and \(\sigma_{w} > 0\) are hyper-parameters.

        For the posterior \(q(\theta; \psi)\) of the meta-parameter \(\theta\), it is often assumed to be a Dirac delta function: \(q(\theta; \psi) = \delta(\theta - \psi)\)~\citep{finn2017model, ravi2018amortized, nguyen2020uncertainty}. Such assumption would, however, make the KL divergence \(\mathrm{KL} \left[ q(\theta; \psi) || p(\theta) \right]\) in Theorem~\ref{theorem:pac_bayes_meta_learning} undefined. In this paper, \(q(\theta; \psi)\) is assumed to be a multivariate normal distribution with a fixed diagonal covariance matrix, denoted as:
        \begin{equation}
            q(\theta; \psi) = \mathcal{N} \left( \theta; \bm{\mu}_{\theta}, \sigma_{0} \mathbf{I} \right),
            \label{eq:q_theta}
        \end{equation}
        where \(\sigma_{0}\) is a hyper-parameter. Such notation also means that \(\psi = \bm{\mu}_{\theta}\).

        The only distribution left is the task-specific \say{posterior} \(q(\mathbf{w}_{i}; \lambda_{i})\). In general, \(q(\mathbf{w}_{i}; \lambda_{i})\) can be modelled following one of the two general types: prescribed and implicit \cite{diggle1984monte}. For example, ABML~\cite{ravi2018amortized} and VAMPIRE~\cite{nguyen2020uncertainty} are prescribed approaches where \(q(\mathbf{w}_{i}; \lambda_{i})\) is assumed to be a multivariate normal distribution with a diagonal covariance matrix. Such approximation is, however, inexpressive, resulting in a poor estimation. In contrast, implicit modelling only has access to the samples generated from the distribution of interest without assuming any analytical form of such distribution, e.g. the generator in generative adversarial networks~\citep{goodfellow2014generative}. In this paper, we use the implicit modelling approach to expressively represent \(q(\mathbf{w}_{i}; \lambda_{i})\). 

        The distribution \(q(\mathbf{w}_{i}; \lambda_{i})\) is now defined at a more fundamental level whereby data is generated through a stochastic mechanism without specifying a parametric distribution. A parameterised model (i.e., a generator \(G\) represented by a deep neural network) is used to model the sample generation:
		\begin{equation}
    		\mathbf{w}_{i} \sim q(\mathbf{w}_{i}; \lambda_{i}) \Leftrightarrow \mathbf{w}_{i} = G(\mathbf{z}; \lambda_{i}), \mathbf{z} \sim p(\mathbf{z}),
		\label{eq:generator}
		\end{equation}
		\orange{where \(\mathbf{z} \in \mathcal{Z} \subseteq \mathbb{R}^{z}\) is the latent noise sampled from a latent noise distribution \(p(\mathbf{z})\), which is often selected as the uniform in \([0, 1]^{z}\) or the standard normal distribution. In our implementation, we observe that due to the unconstrained support space of the standard normal distribution, latent noise sampled from such distribution may vary drastically, resulting in a large variation of the generated task-specific model parameter \(\mathbf{w}_{i}\) and potentially, making the training difficulty, especially at the beginning of the training. We, therefore, use the uniform distribution on \([0, 1]^{z}\) as the latent noise distribution \(p(\mathbf{z})\) to bound the support space \(\mathcal{Z}\) of the latent noise to make the training more stable.}

		\magenta{Due to the nature of implicit models, the KL divergence term \(\mathrm{KL} \left[ q(\mathbf{w}_{i}; \lambda_{i}) || p(\mathbf{w}_{i}) \right]\) in the lower-level and \(\mathrm{KL} \left[ q(\mathbf{w}_{i}; \lambda_{i}^{*}) || p(\mathbf{w}_{i}) \right]\) in the upper-level of the bi-level optimisation objective cannot be evaluated either analytically or symbolically. We, therefore, propose to employ the lower-bound of the KL divergence shown in Lemma~\ref{lmm:compression} to estimate the value of the KL divergence to train the meta-learning model.}
		
		\begin{restatable}[Compression lemma~\citep{banerjee2006bayesian}]{lemma}{KLLowerBound}
        \label{lmm:compression}
            \magenta{For any measurable function \(\phi(h)\) on a set of predictors under consideration \(\mathcal{H}\), and any distributions \(P\) and \(Q\) on \(\mathcal{H}\), the following holds:}
            \begin{equation*}
                \mathbb{E}_{Q} \left[ \phi(h) \right] - \ln \mathbb{E}_{P} \left[ e^{\phi(h)} \right] \le \mathrm{KL} \left[ Q \Vert P \right].
            \end{equation*}
            
            Further,
            \begin{equation*}
                \sup_{\phi} \mathbb{E}_{Q} \left[ \phi(h) \right] - \ln \mathbb{E}_{P} \left[ e^{\phi(h)} \right] = \mathrm{KL} \left[ Q \Vert P \right].
            \end{equation*}
        \end{restatable}
        \begin{proof}
            Please refer to the proof in \appendixautorefname~\ref{apdx:auxiliary_lemmas}.
        \end{proof}

        \magenta{To model the function \(\phi: \mathcal{W} \to \mathbb{R}\) in Lemma~\ref{lmm:compression} to estimate the lower-bound of KL divergences, we use a neural network with parameter \(\omega_{i}\). The objective to obtain \(\phi\) is to maximise the left-hand side in Lemma~\ref{lmm:compression}, which can be expressed as:}
        \begin{equation}
            \omega_{i}^{*} = \arg\max_{\omega_{i}} \mathbb{E}_{q(\mathbf{w}_{i}; \lambda_{i})} \left[ \phi(\mathbf{w}_{i}; \omega_{i}) \right] - \ln \mathbb{E}_{p(\mathbf{w}_{i})} \left[ e^{\phi(\mathbf{w}_{i}; \omega_{i})} \right].
            \label{eq:KL_divergence_lower_bound_objective}
        \end{equation}
        
        \magenta{The KL divergence between the task-specific posterior \(q(\mathbf{w}_{i}; \lambda_{i})\) and prior \(p(\mathbf{w}_{i})\) can then be estimated as:}
        \begin{equation}
            \begin{aligned}[b]
                \mathrm{KL} \left[ q(\mathbf{w}_{i}; \lambda_{i}) || p(\mathbf{w}_{i}) \right] & = \mathbb{E}_{q(\mathbf{w}_{i}; \lambda_{i})} \left[ \phi(\mathbf{w}_{i}; \omega_{i}^{*}) \right] \\
                & \quad - \ln \mathbb{E}_{p(\mathbf{w}_{i})} \left[ e^{\phi(\mathbf{w}_{i}; \omega_{i}^{*})} \right].
            \end{aligned}
            \label{eq:KL_divergence_lower_bound}
        \end{equation}
		
		One problem that arises when estimating the losses in both the lower- and upper-level of the meta-learning objective is to learn a different \(\omega_{i}\) for each different task \(\mathcal{T}_{i}\) by training the neural network \(\phi\) from scratch. The downside of such naive implementation is the significant increase in training time. We, therefore, propose to learn a good initialisation of \(\omega_{i}\) using MAML~\cite{finn2017model} to reduce the time of the KL divergence estimation. With this assumption, we define \(\omega_{0}\) as the meta-parameters (or initialisation) of \(\omega_{i}\). Within each task, we initialise \(\omega_{i}\) at \(\omega_{0}\) and optimise the loss in \eqref{eq:KL_divergence_lower_bound_objective} w.r.t. \(\omega_{i}\) using gradient descent. \(\omega_{0}\) is then obtained by optimising the validation loss evaluated on \(\omega_{i}^{*}\) w.r.t. \(\omega_{0}\).
		
		The proposed meta-learning method, therefore, consists of two parameters of interest: the hyper-meta-parameter \(\psi\) and the meta-parameter \(\omega_{0}\) of the \(\phi\)-network. Each of the parameters is optimised following their corresponding bi-level optimisation objective functions similar to \eqref{eq:meta_learning_objective}. Algorithm~\ref{algorithm:simpa_train} shows the training procedure of the proposed approach. In addition, we name the proposed approach as \emph{statistical implicit PAC-Bayes meta-learning} or SImPa for short, to simplify the text when comparing to other meta-learning methods.

		\begin{algorithm}[t]
    		\caption{SImPa}
    		\label{algorithm:simpa_train}

    		\begin{algorithmic}[1]
		        \Procedure{train}{}
		            \State initialise hyper-meta-parameter \(\psi = \bm{\mu}_{\theta}\)
		            \State initialise \(\phi\)-network meta-parameter \(\omega_{0}\)
		            \While{\(\psi\) and \(\omega_{0}\) not converged}
		                \State sample: \(\theta \sim \mathcal{N} \left( \theta; \bm{\mu}_{\theta}, \sigma_{0} \mathbf{I} \right)\)
		                \State sample a mini-batch of \(T\) tasks
		                \For{each task \(\mathcal{T}_{i}\)}
		                    \State \(\lambda_{i}, \omega_{i} \gets\) \Call{optimise lower-level}{$\theta, \omega_{0}, \mathcal{T}_{i}$}
		                    \State use \(\lambda_{i}\) to calculate PAC-Bayes upper-bound in Theorem~\ref{theorem:pac_bayes_meta_learning} \label{algorithm_step:validation_loss_psi}
		                    \State use \(\omega_{i}\) to calculate the KL lower-bound in \eqref{eq:KL_divergence_lower_bound}\label{algorithm_step:validation_KL_lower_bound}
		                \EndFor
		                \State \(\psi \gets \mathrm{SGD} \left( \text{average of \(T\) losses obtained in step \ref{algorithm_step:validation_loss_psi}} \right)\)
		                \State \(\omega_{0} \gets \mathrm{SGA}\)(average of \(T\) KL lower-bounds obtained in step \ref{algorithm_step:validation_KL_lower_bound})
		                \Statex \Comment{SGD/SGA = stochastic gradient descent/ascent}
		            \EndWhile
		            \State \textbf{return} \(\psi\) and \(\omega_{0}\)
		        \EndProcedure
                \Statex
                \Procedure{optimise lower-level}{$\theta, \omega_{0}, \mathcal{T}_{i}$}
                    \State initialise task-specific parameter: \(\lambda_{i} \gets \theta\)
                    \State initialise task-specific \(\phi\)-net: \(\omega_{i} \gets \omega_{0}\)
                    \State \(\omega_{i}^{*} \gets \arg\max_{\omega_{i}} \text{KL lower-bound in \eqref{eq:KL_divergence_lower_bound_objective}}\)
                    \State use \(\omega_{i}^{*}\) to calculate KL divergence in Eq.~\eqref{eq:KL_divergence_lower_bound}
                    \State \(\lambda_{i}^{*} \gets \arg\min_{\lambda_{i}} \text{PAC-Bayes upper-bound in Theorem~\ref{theorem:pac_bayes_single_task}}\)
                    \State \textbf{return} \(\lambda_{i}^{*}\) and \(\omega_{i}^{*}\)
                \EndProcedure
		    \end{algorithmic}
		\end{algorithm}

		One potential drawback of the implicit modelling approach is the curse of dimensionality, resulting in a computationally expensive training process. This is an active research question when dealing with generative models. This issue can be addressed by encoding the high-dimensional data, such as images, to a feature embedding space by supervised-learning on the same training data set. For example, in one of our experiments, we also show how to use image features extracted from a wide-residual network trained on tiered-ImageNet training set~\citep{rusu2019meta} to mitigate this issue. This strategy reduces the dimension of the input space, leading to smaller task-specific model parameter \(\mathbf{w}_{i}\), and eventually decreasing the size of the generator. The trade-off lies in the possibility of losing relevant information that can affect the performance on hold-out tasks. 
		
		
		It is also worth noting that our proposed method is easier to train than prior probabilistic meta-learning methods~\cite{finn2018probabilistic, ravi2018amortized} because we no longer need to tune the weighting factor of \(\mathrm{KL} \left[ q(\mathbf{w}_{i}; \lambda_{i}) || p(\mathbf{w}_{i}) \right]\) in both levels of the bi-level optimisation objective. Although weighting such KL divergence terms can be justified by casting the optimisation in each level of the meta-learning objective to a constrained optimisation as shown in \(\beta\)-VAE~\cite{higgins2016beta}, the weighting factor in such case is the corresponding Lagrange multiplier of the constrained optimisation. Thus, simply setting that weighting factor as a \say{tunable} hyper-parameter may result in a sub-optimal solution. In contrast, our proposed approach does not need to re-weight the KL divergences. The trade-off of our approach lies in the need to set the confidence parameter \(\varepsilon\), but tuning \(\varepsilon\) is arguably more intuitive than tuning the correct weighting factor for the KL divergence terms done in previous works.
\section{Experimental Evaluation}
\label{sec:experiments}
    \begin{figure*}[t]
		\centering
		\begin{subfigure}[t]{0.24\linewidth}
			\centering
			\includegraphics[width = \linewidth]{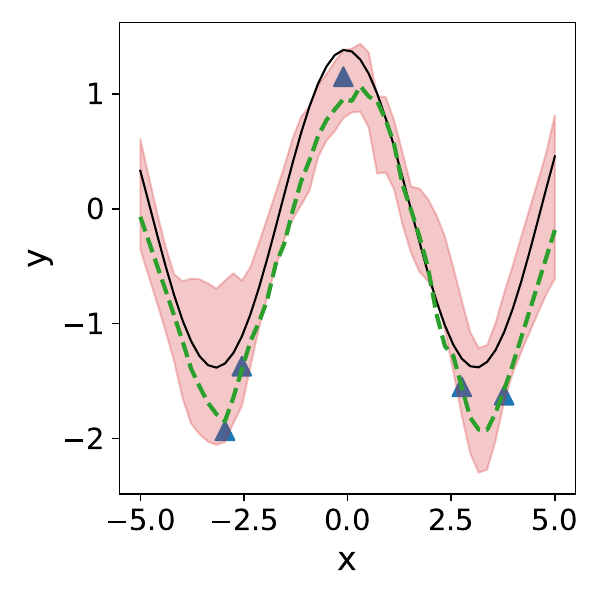}
		\end{subfigure}
		\begin{subfigure}[t]{0.24\linewidth}
			\centering
			\includegraphics[width = \linewidth]{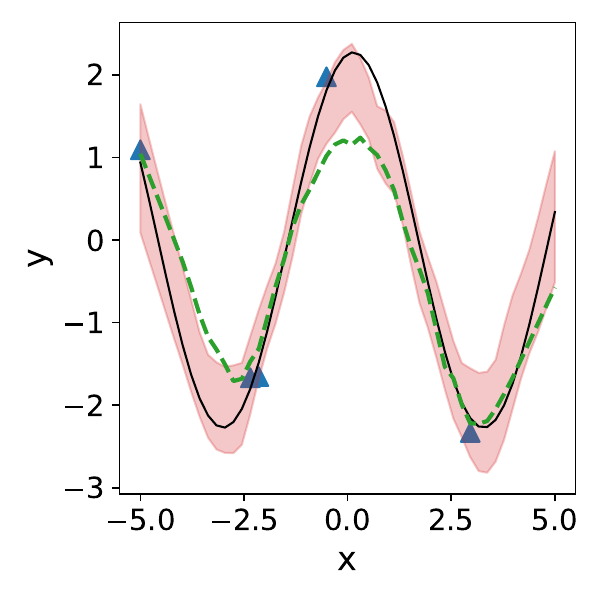}
		\end{subfigure}
		\begin{subfigure}[t]{0.24\linewidth}
			\centering
			\includegraphics[width = \linewidth]{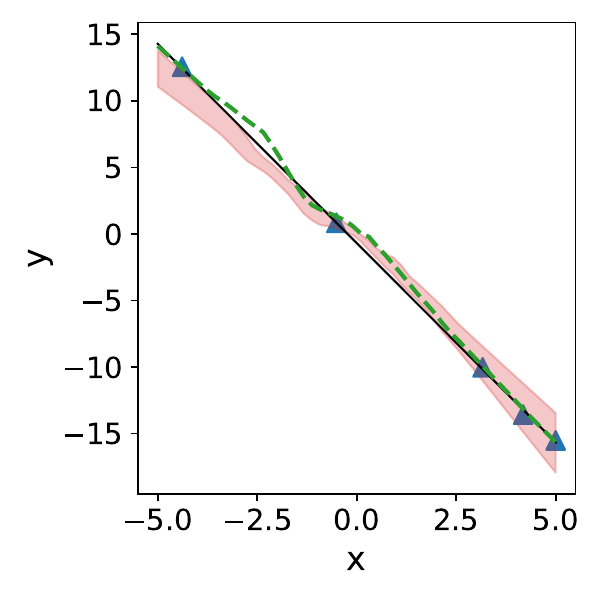}
		\end{subfigure}
		\begin{subfigure}[t]{0.24\linewidth}
			\centering
			\includegraphics[width = \linewidth]{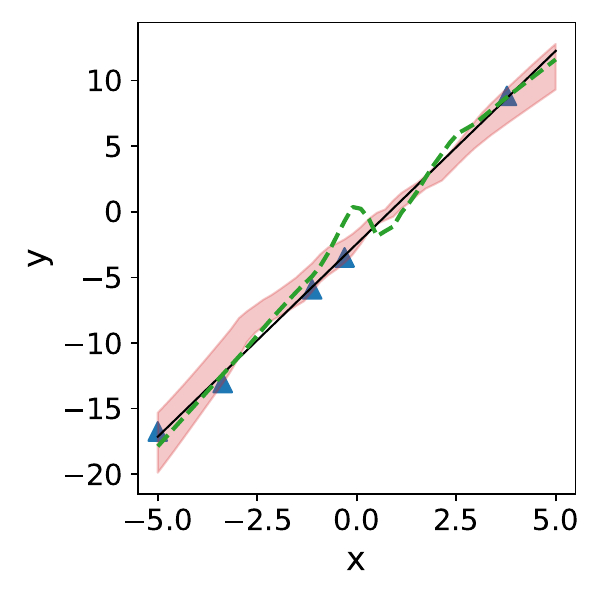}
		\end{subfigure}
		\\
		\includegraphics[height=1.5em]{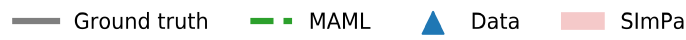}
		\caption{SImPa and MAML are compared in a regression problem when training is based on multi-modal data -- half of the tasks are generated from sinusoidal functions, and the other half are from linear functions. The shaded area is the prediction made by SImPa \(\pm\) 3\(\times\) standard deviation.}
		\label{fig:regression_visualisation}
	\end{figure*}
        
        \begin{figure*}[t]
        	\centering
        	\begin{subfigure}[b]{0.25 \linewidth}
        	    \centering
                \begin{tikzpicture}
                    \pgfplotstableread[col sep=&, row sep=\\, header=true]{
                        method & nll\\
                        MAML & 2.32\\
                        PLATIPUS & 2.52\\
                        BMAML & 1.45\\
                        ABML & 2.65\\
                        SImPa & 1.58\\
                    } \myTable

                    \pgfplotstablegetrowsof{\myTable}
                    \pgfmathsetmacro{\NumRows}{\pgfplotsretval-1}

                    \begin{axis}[
                        height = 0.9 \linewidth,
                        width = 0.9 \linewidth,
                        xbar = 0pt,
                        xmin = 1.25,
                        xmax = 2.75,
                        xticklabel style = {font=\footnotesize},
                        xlabel={Mean squared error},
                        xlabel style = {font=\footnotesize},
                        ytick=data,
                        yticklabels from table={\myTable}{method},
                        yticklabel style = {font=\footnotesize, align=left},
                        scale only axis,
                        enlarge x limits=auto,
                        enlarge y limits=0.1,
                        grid=major,
                        grid style={dotted, thick},
                        every axis plot/.append style={fill=NavyBlue, draw=none}
                    ]
                        \addplot[] table [y expr=\NumRows - \coordindex, x=nll]{\myTable};
                    \end{axis}
                \end{tikzpicture}
        	    \caption{MSE on hold-out tasks}
        	    \label{fig:SImPa_regression_NLL}
        	\end{subfigure}
        	\hspace{2em}
        	\hfill
        	\begin{subfigure}[b]{0.25\linewidth}
        		\centering
                \begin{tikzpicture}
                    \pgfplotstableread[col sep=comma, row sep=\\, header=true]{
                        x,MAML,PLATIPUS,BMAML,ABML,SImPa\\
                        0,0.4860619902610778809,0.36753817,0.32526379,0.40130252,0.20724078\\
                        0.1,0.4860619902610778809,0.39879262,0.38103479,0.43856199,0.32954753\\
                        0.2,0.4860619902610778809,0.42135352,0.42053903,0.464544,0.3804987\\
                        0.3,0.4860619902610778809,0.44004638,0.45217654,0.48548708,0.42313374\\
                        0.4,0.4860619902610778809,0.4571563,0.47961517,0.50396608,0.46094437\\
                        0.5,0.4860619902610778809,0.47404441,0.50541934,0.52153283,0.49935024\\
                        0.6,0.4860619902610778809,0.49129097,0.53134508,0.53895331,0.53686566\\
                        0.7,0.4860619902610778809,0.50953219,0.55830806,0.55695198,0.57479658\\
                        0.8,0.4860619902610778809,0.52991487,0.58850907,0.57726778,0.61686458\\
                        0.9,0.4860619902610778809,0.55434502,0.62519009,0.60150843,0.67059377\\
                        1,0.4860619902610778809,0.59050844,0.67733319,0.63569099,0.78212945\\
                    } \calibrationTable
                    
                    \begin{axis}[
                        height = 0.9 \linewidth,
                        width = 0.9 \linewidth,
                        xlabel={Observed probability},
                        xlabel style={font=\footnotesize},
                        xticklabel style = {font=\footnotesize},
                        ylabel={Expected probability},
                        ylabel style = {font=\footnotesize, yshift=0em},
                        yticklabel style = {font=\footnotesize},
                        scale only axis,
                        table/x=x,
                        legend entries = {MAML, PLATIPUS, BMAML, ABML, SImPa},
                        legend style = {draw=none, font=\scriptsize, yshift=0.25em, /tikz/every even column/.append style={column sep=0.5em}},
                        legend cell align = {left},
                        legend columns = 2,
                        legend image post style={scale=0.5},
                        legend pos = north west
                    ]
                        \addplot[mark=none, NavyBlue, dashdotted, very thick] table[y={MAML}]{\calibrationTable};
                        \addplot[mark=none, PineGreen, dashed, very thick] table[y= {PLATIPUS}]{\calibrationTable};
                        \addplot[mark=none, BurntOrange, densely dashdotted, very thick] table[y = {BMAML}]{\calibrationTable};
                        \addplot[mark=none, RoyalPurple, dash pattern=on 4pt off 1pt, very thick] table[y = {ABML}]{\calibrationTable};
                        \addplot[mark=none, BrickRed, solid, very thick] table[y = {SImPa}]{\calibrationTable};
                        
                        \addplot[mark=none, Black, densely dashed] table[y = {x}]{\calibrationTable};
                    \end{axis}
                \end{tikzpicture}
        		\caption{Reliability diagram}
        		\label{fig:SImPa_regression_reliability_chart}
        	\end{subfigure}
        	\hfill
        	\begin{subfigure}[b]{0.25\linewidth}
                \begin{tikzpicture}
                    \pgfplotstableread[col sep=comma, row sep=\\, header=true]{
                        method,ECE,MCE\\
                        MAML,0.2739943645217202,0.5139380097389221\\
                        PLATIPUS,0.2213864627272727,0.40949156\\
                        BMAML,0.18939665181818183,0.32526379\\
                        ABML,0.21863836454545452,0.40130252\\
                        SImPa,0.14734225818181818,0.22954752999999997\\
                    } \ceTable
                    
                    \pgfplotstablegetrowsof{\ceTable}
                    \pgfmathsetmacro{\NumRows}{\pgfplotsretval-1}

                    \begin{axis}[
                        width = 0.9 \linewidth,
                        height = 0.9 \linewidth,
                        xbar=0pt,
                        ytick=data,
                        yticklabels from table={\ceTable}{method},
                        yticklabel style = {font=\footnotesize},
                        xticklabel style = {font=\footnotesize},
                        xlabel style = {font=\footnotesize},
                        xlabel = {Calibration error},
                        scale only axis,
                        enlarge x limits=auto,
                        enlarge y limits=0.15,
                        grid=major,
                        grid style={dotted, thick},
                        every axis plot/.append style={fill, draw=none},
                        legend entries = {ECE,MCE},
                        legend style = {draw=none, font=\scriptsize, yshift=0em, /tikz/every even column/.append style={column sep=0.5em}},
                        legend cell align = {left},
                        legend columns = 2,
                        legend image post style={scale=1},
                        legend pos = south east
                    ]
                        \addplot[mark=none, NavyBlue] table [y expr=\NumRows - \coordindex, x = {ECE}] {\ceTable};
                        \addplot[mark=none, BrickRed] table [y expr=\NumRows - \coordindex, x = {MCE}] {\ceTable};
                    \end{axis}
                \end{tikzpicture}
        		\caption{ECE and MCE}
        		\label{fig:SImPa_regression_ECE_MCE}
        	\end{subfigure}
        	\hfill
        	\hphantom{x}

        	\caption{Quantitative comparison between various probabilistic meta-learning approaches averaged over 1000 unseen tasks shows that SImPa has a comparable MSE error and the smallest calibration error.}
        	\label{fig:regression_calibration}
        \end{figure*}
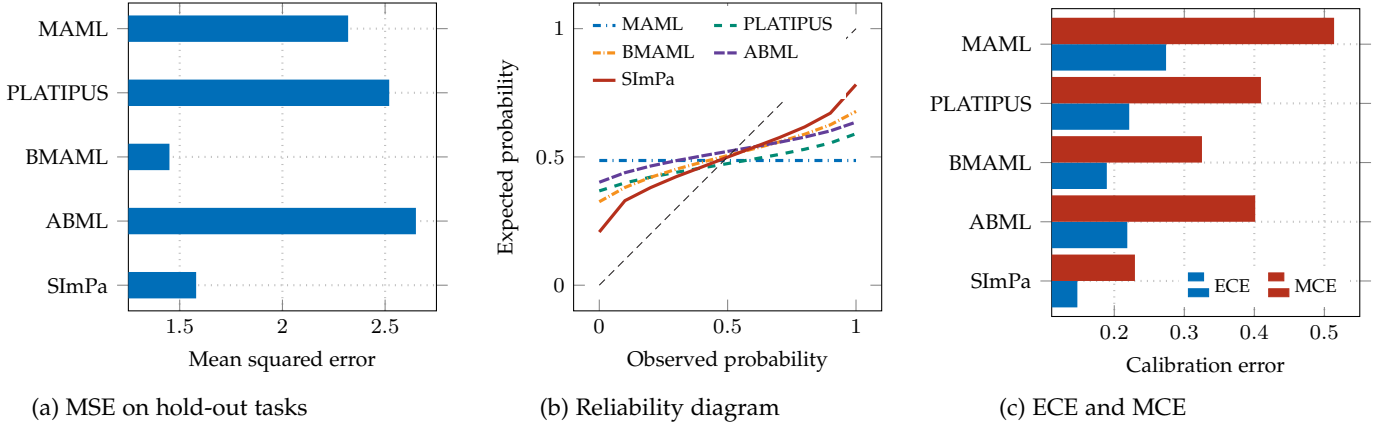

    In this section, SImPa is evaluated on few-shot regression and classification problems and compared to common meta-learning baselines, such as MAML~\citep{finn2017model}, ABML~\citep{ravi2018amortized}, PLATIPUS~\citep{finn2018probabilistic} and BMAML~\citep{yoon2018bayesian}.

    In the experiments, the loss function \(\ell\) is assumed to be the negative log-likelihood (NLL): \(-\ln p(y_{ij} | \mathbf{x}_{ij}; \mathbf{w}_{i})\). In particular, the NLL is the mean-squared error (MSE) for regression and cross-entropy for classification. Following the assumption of bounded losses made in Section~\ref{sec:methodology}, the data-related losses in both the lower- and upper-level, \(\ell(\mathbf{x}_{ij}, y_{ij}; \mathbf{w}_{i})\) are clipped to \([0, 1]\). In addition, Monte Carlo (MC) sampling is used to evaluate the expectation over \(q(\theta; \psi)\) and \(q(\mathbf{w}_{i}; \lambda_{i})\). In particular, one sample of meta-parameter \(\theta\) and one sample of task-specific parameter \(\mathbf{w}_{i}\) are used in training, while these numbers are one and thirty-two in testing, respectively. 
    The asymmetric choice of those hyper-parameters is to optimise running time in training, while maximising the prediction performance in evaluation. 
    For the hyper-parameters defined in \eqref{eq:prior_theta}, \eqref{eq:prior_w} and \eqref{eq:q_theta}, \(\sigma_{\theta} = \sigma_{\mathbf{w}} = 1\) and \(\sigma = 10^{-6}\). In addition, the confidence parameters is selected as \(\varepsilon = \varepsilon_{i} = 0.1, \forall i \in \{1, \ldots, T\}\). The number of tasks per an update of the parameters of interest is \(T = 20\).

    In term of generating \(\mathbf{w}_{i}\) for each task \(\mathcal{T}_{i}\), we use latent noise vectors \(\mathbf{z}\) sampled from the uniform distribution in \([0, 1]^{128}\). The generator in \eqref{eq:generator} is modelled by a fully-connected neural network with 2 hidden layers. Each of these layers consists of 256 and 512 hidden units, respectively, and is activated by rectified linear unit without any batch normalisation. The output of the final layer is then activated by hyperbolic tangent function to constrain the parameter of the base network, avoiding the loss value from exploding. The \(\phi\)-network has an \say{inverted} architecture of the generator, which consists of 3-hidden layers. These layers consist of 512, 256 and 128 hidden units, respectively, and are also activated by rectified linear unit without any batch normalisation. Adam optimiser~\citep{kingma2015adam} is used to optimise both the hyper-meta-parameter \(\psi\) and \(\omega_{0}\) with the same learning rate of \(10^{-4}\). To train the \(\phi\)-network for each task, 512 MC samples of the base network parameters are sampled from both \(q(\mathbf{w}_{i}; \lambda_{i})\) and \(p(\mathbf{w}_{i})\) to evaluate the lower-bound of the KL divergence in \eqref{eq:KL_divergence_lower_bound_objective}. To optimise~\eqref{eq:KL_divergence_lower_bound_objective} w.r.t. \(\omega_{i}\), gradient descent is used with a learning rate of \(10^{-4}\). 

    \subsection{Regression}
    \label{sec:regression}
    	The experiment in this subsection is a multi-modal task distribution where half of the data is generated from sinusoidal functions, while the other half is from linear functions \cite{finn2018probabilistic}. The sinusoidal function used in this experiment is in the form \(y = A\sin(x + \Phi) + \epsilon\), where \(A\) and \(\Phi\) are uniformly sampled from \([0.1, 5]\) and \([0, \pi]\), respectively, while the linear function considered is in the form \({y=ax + b + \epsilon}\), where \(a\) and \(b\) are randomly sampled from \([-3, 3]\). The noise \(\epsilon\) is sampled from \(\mathcal{N}(0, 0.3^{2})\). The experiment is carried out under the 5-shot setting: \(m_{i}^{(t)} = 5\), and the validation set \(S_{i}^{(v)}\) consists of \(m_{i}^{(v)} = 50\) data points. 

        Similar to existing works in the literature~\citep{finn2018probabilistic}, the base network to solve each regression task is a fully-connected network with 2 hidden layers, where each layer has 40 hidden units. Linear rectifier function is used as activation function, and no batch-normalisation is used. Gradient descent is used as the optimiser for the lower-level optimisation with learning rate fixed at \(10^{-3}\) and iterated 5 times.
        
        As shown in \figureautorefname{~\ref{fig:regression_visualisation}}, SImPa is able to vary the prediction variance, especially when there is more uncertainty in the training data, while MAML can only output a single value at each data point. For a quantitative comparison, we train many probabilistic meta-learning methods, including PLATIPUS~\cite{finn2018probabilistic}, BMAML~\cite{yoon2018bayesian} and ABML~\cite{ravi2018amortized}, in the same regression problem. Here, BMAML consists of 10 particles trained without Chaser Loss. As shown in \figureautorefname~\ref{fig:SImPa_regression_NLL}, SImPa achieves much smaller MSE, comparing to MAML, PLATIPUS and ABML, and comparable NLL to the non-parametric BMAML when being evaluated on the same hold-out tasks.

        To further evaluate the predictive uncertainty, we employ the reliability diagram based on the quantile calibration for regression~\cite{song19distribution}. The reliability diagram shows a correlation between predicted and actual probability. A perfectly calibrated model will have its predicted probability equal to the actual probability, and hence, align well with the diagonal \(y = x\). The results in \figureautorefname~\ref{fig:SImPa_regression_reliability_chart} show that the model trained with SImPa achieves the best calibration among all the methods considered. Due to the nature of a deterministic approach, MAML~\cite{finn2017model} is represented as a horizontal line, resulting in a poorly calibrated model. The two probabilistic meta-learning methods, PLATIPUS and ABML, perform better than MAML; however, the averaged slopes of their performance curves are quite close to MAML, implying that their multivariate normal posteriors of task-specific model parameters have small covariance diagonal values. This may be caused by their exclusive reliance on less-expressive multivariate normal distributions with diagonal covariance matrices. The performance of BMAML is slightly better than PLATIPUS and ABML due to its non-parameteric modelling approach. In contrast, SImPa employs a much richer variational distribution \(q(\mathbf{w}_{i}; \lambda_{i})\) for task specific parameters, and therefore, produces a model with better calibration. For another quantitative comparison, we plot the expected calibration error (ECE)~\cite{guo2017oncalibration}, which is the weighted average of the absolute errors measuring from the diagonal, and the maximum calibration error (MCE)~\cite{guo2017oncalibration}, which returns the maximum of absolute errors in \figureautorefname~\ref{fig:SImPa_regression_ECE_MCE}. Overall, SImPa outperforms all of the state-of-the-art methods in both ECE and MCE.
        
    \subsection{Few-shot classification}
    \label{sec:classification}
        We evaluate SImPa on the \(N\)-way \(k\)-shot setting, where a meta learner is trained on many related tasks containing \(N\) classes with \(k\) examples per class (\({m_{i}^{(t)} = kN}\)). The evaluation is carried out by comparing the results of SImPa against the results of state-of-the-art methods on three popular few-shot learning benchmarking data sets: Omniglot~\cite{lake2015human}, mini-ImageNet~\cite{vinyals2016matching,ravi2017optimization} and tiered-ImageNet~\cite{ren2018meta}.
	    
	    Omniglot dataset consists of 50 different alphabets with a total of 1623 characters drawn online via Amazon's Mechanical Turk by 20 different people. Hence, Omniglot is often considered as a \say{transposed} MNIST since Omniglot has many classes, but each class has 20 images. We follow the original train-test split where 30 alphabets are used for training, while the other 20 alphabets are used for testing. To be consistent with previous evaluations, we pre-process by down-sampling all images to 28-by-28 pixels. No data augmentation, such as rotation, is used. Note that for the task formation, many existing meta-learning methods in the literature use non-standard train-test split where characters of all 50 alphabets are mixed, and randomly split. This splitting potentially forms easier tasks since knowing a character in an alphabet might help to classify other characters within that same alphabet. Moreover, the mixed and random split is different from evaluation to evaluation, making it challenging to fairly compare different meta-learning methods.
        
        Mini-ImageNet~\cite{vinyals2016matching} is another dataset used to evaluate classification performance between different meta-learning methods. The dataset consists of 100 classes, where each class contains 600 colour images taken from ImageNet~\cite{ILSVRC15}. We follow the standard train-test split which uses 64 classes for training, 16 classes for validation, and 20 classes for testing~\cite{ravi2017optimization}. The images in the dataset are pre-processed by down-sampling to 84-by-84 pixels before any training is carried out. No data augmentation, such as image flipping or rotation, is used.
        
        Tiered-ImageNet is one of the largest subsets of ImageNet, which consists of total 608 classes grouped into 34 high-level categories~\cite{ren2018meta}. Tiered-ImageNet is often used as a benchmark for large-scaled few-shot learning. We also follow the standard train-test split that consists of 20 categories for training, 6 categories for validation, and 8 categories for testing. In addition, our evaluation is carried out by employing the features extracted from a residual network trained on the data and classes from the training set~\cite{rusu2019meta}.

        \begin{table}[t!]
	    	\begin{center}
	    		\begin{small}
	    			\begin{tabular}{l c c}
	    				\toprule
	    				\bfseries METHOD & \bfseries 1-SHOT & \bfseries 5-SHOT\\
	    				\midrule
	    				\midrule
	    				\multicolumn{3}{l}{\textbf{Omniglot~\cite{lake2015human} - standard 4-block CNN} }\\
	    				\midrule
	    				MAML~\cite{finn2017model} & \(97.143 \pm 0.005\) & \\
	    				Prototypical nets~\cite{snell2017prototypical} & \(96.359 \pm 0.006\) & \\
	    				BMAML~\cite{yoon2018bayesian} & \(94.104 \pm 0.008\) & \\
	    				ABML~\cite{ravi2018amortized} & \(97.281 \pm 0.004\) & \\
	    				\rowcolor{gray!30} \textbf{SImPa} & \textbf{98.352} \(\pm\) \textbf{0.005} & \\
	    				\midrule
	    				\midrule
	    				\multicolumn{3}{l}{\textbf{Mini-ImageNet~\cite{ravi2017optimization} - standard 4-block CNN} }\\
	    				\midrule
	    				Matching nets \cite{vinyals2016matching} & 43.56 $\pm$ 0.84 & 55.31 $\pm$ 0.73 \\
	    				Meta-learner LSTM \cite{ravi2017optimization} & 43.44 $\pm$ 0.77 & 60.60 $\pm$ 0.71 \\
	    				MAML \cite{finn2017model} & 48.70 $\pm$ 1.84 & 63.15 \(\pm\) 0.91 \\
	    				Prototypical nets \cite{snell2017prototypical}\tablefootnote{Trained on 30-way 1-shot setting} & 49.42 $\pm$ 0.78 & \textbf{68.20 $\pm$ 0.66} \\
	    				LLAMA \cite{grant2018recasting} & 49.40 $\pm$ 1.83 & \_ \\
	    				PLATIPUS \cite{finn2018probabilistic} & 50.13 $\pm$ 1.86 & \_ \\
	    				ABML \cite{ravi2018amortized} & 45.00 $\pm$ 0.60 & \_ \\
	    				\rowcolor{gray!30}\textbf{SImPa} & \textbf{51.72} \(\pm\) \textbf{0.48} & 63.49 \(\pm\) 0.40 \\
	    				\midrule
	    				\midrule
	    				\multicolumn{3}{l}{\textbf{Mini-ImageNet~\cite{ravi2017optimization} - non-standard network}}\\
	    				\midrule
	    				Relation nets \cite{Sung_2018_CVPR} & 50.44 $\pm$ 0.82 & 65.32 $\pm$ 0.70 \\
	    				VERSA \cite{gordon2018metalearning} & 53.40 $\pm$ 1.82 & 67.37 $\pm$ 0.86 \\
	    				SNAIL \cite{mishra2018simple} & 55.71 $\pm$ 0.99 & 68.88 $\pm$ 0.92 \\
	    				adaResNet \cite{munkhdalai2018rapid} & 56.88 $\pm$ 0.62 & 71.94 $\pm$ 0.57 \\
	    				TADAM \cite{oreshkin2018tadam} & 58.50 $\pm$ 0.30 & 76.70 $\pm$ 0.30 \\
	    				LEO \cite{rusu2019meta} & 61.76 $\pm$ 0.08 & \bfseries 77.59 $\pm$ 0.12 \\
	    				LGM-Net \cite{li2019lgm} & \bfseries 69.13 \(\pm\) 0.35 & 71.18 \(\pm\) 0.68 \\
	    				\rowcolor{gray!30} \textbf{SImPa}\tablefootnote{\label{ftnt:extracted_features}Use extracted features~\cite{rusu2019meta} as input} &  62.85 \(\pm\) 0.56  & \bfseries 77.65 \(\pm\) 0.50 \\
	    				\bottomrule
	    				\toprule
	    				\multicolumn{3}{l}{\textbf{Tiered-ImageNet}~\cite{ren2018meta} non-standard network}\\
	    				\midrule
	    				MAML \cite{liu2018transductive} & \(51.67 \pm 1.81\) & \(70.30 \pm 0.08\) \\
	    				Proto. Nets \cite{ren2018meta} & \(53.31 \pm 0.89\) & \(72.69 \pm 0.74\) \\
	    				Relation Net \cite{liu2018transductive} & \(54.48 \pm 0.93\) & \(71.32 \pm 0.78\) \\
	    				Trns. Prp. Nets \cite{liu2018transductive} & \(57.41 \pm 0.94\) & \(71.55 \pm 0.74\) \\
	    				LEO \cite{rusu2019meta} & \(66.33 \pm 0.05\) & \bfseries 81.44 \(\pm\) 0.09 \\
	    				MetaOptNet \cite{lee2019meta} & \(65.81 \pm 0.74\) & \bfseries 81.75 \(\pm\) 0.53 \\
	    				\rowcolor{gray!30}\textbf{SImPa}\textsuperscript{\getrefnumber{ftnt:extracted_features}} & \bfseries 70.26 \(\pm\) 0.35 & 80.15 \(\pm\) 0.28 \\
	    				\bottomrule
	    			\end{tabular}
	    		\end{small}
	    	\end{center}
	    	\caption{The few-shot 5-way classification accuracy results (in percentage, with 95\% confidence interval) of SImPa averaged over 1 million tasks on Omniglot (top), and 600 tasks on mini-ImageNet (middle-top and middle-bottom) and tiered-ImageNet (bottom) datasets. For each experiment, we select the top two methods with highest mean values, and apply t-test with 95 percent confidence. If they are significantly different, we highlight the method with largest mean in bold, otherwise, we highlight both of them.} 
	    	\label{tab:classification_accuracies}
	    \end{table}

        The evaluation is carried out in 2 cases: one with raw image data and the other with 640-dimensional image features extracted from a wide-residual network trained solely on training data~\citep{rusu2019meta}. In the \say{standard} case, the base network is the \say{standard} convolutional network consisting of 4 modules. 
        Each module has a convolutional layer with 32 3-by-3 filters, followed by a batch normalisation, ReLU and 2-by-2 max-pooling. The output of the final module is flatten and connected to a fully connected layer to predict the label of the input image. 
        In the \say{non-standard} case, the base network is a fully-connected network with 2 hidden layers. Each layer consists of 128 and 32 hidden units activated by rectified linear unit without any batch-normalisation.

        We report the classification accuracy of SImPa on these three data sets in \tableautorefname{~\ref{tab:classification_accuracies}}. For Omniglot, we use the published code to reproduce the results for some common meta-learning methods to fairly compare with SImPa. The accuracy averaged over more than 1 million testing tasks show that the proposed SImPa is better than competing meta-learning methods in the literature.
        For mini-ImageNet, SImPa achieves the best empirical results for the 1-shot setting when the base model is the \say{standard} CNN, and for the 5-shot setting when a different network architecture is used. SImPa shows the second best results for the 5-shot setting with the 4-layer CNN and the 1-shot setting with the different network architecture. Note that for the 5-shot setting using standard CNN, Prototypical networks need to train with a much higher \say{way} which is harder to learn, and might help the trained model to perform better on easier tasks with lower \say{way}. 
        For tiered-ImageNet, SImPa outperforms the current state-of-the-art in 1-shot setting, while being comparable in 5-shot setting. To obtain a fairer comparison, we re-run MAML on the image data of mini-ImageNet using a ResNet10, which has about 5 million parameters (ours has about 8 millions parameters). However, MAML, with and without L2 regularisation, over-fits the training data (our best result for MAML was 89\% accuracy on train, while only 42\% on test). This known issue of overfitting when using larger networks in MAML was mentioned in the MAML paper \cite[Section 5.2]{finn2017model}. We also try a similar model for ABML~\cite{ravi2018amortized}, but observed no improvement.
        
        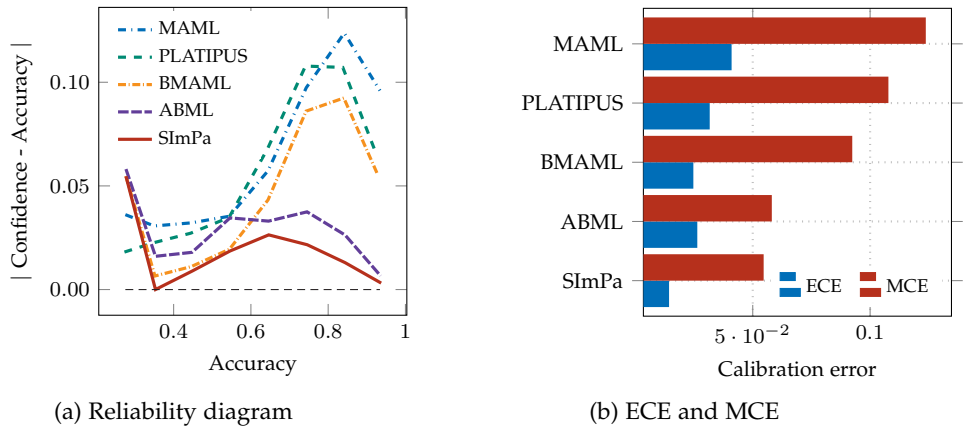
\begin{figure*}[t]
        	\centering
        	\hphantom{dummy}
        	\hfill
        	\begin{subfigure}[b]{0.25\linewidth}
        	\centering
                \begin{tikzpicture}
                    \pgfplotstableread[col sep=comma, row sep=\\, header=true]{
                        confidence,accuracy,numSamples\\
                        0.27493931,0.31107797,0.13140111\\
                        0.34954532,0.38023622,0.36855908\\
                        0.44475241,0.47691715,0.24892092\\
                        0.5442784,0.57964064,0.13027285\\
                        0.64410535,0.70135895,0.06720232\\
                        0.74340833,0.84081141,0.03654821\\
                        0.84173447,0.96546466,0.02335015\\
                        0.93409063,1.03007947,0.02097659\\
                    } \mamlClassification
                    
                    \pgfplotstableread[col sep=comma, row sep=\\, header=true]{
                        confidence,accuracy,numSamples\\
                        0.27255475,0.2906478,0.18670279\\
                        0.34773938,0.37031527,0.39461558\\
                        0.44396027,0.47114644,0.23120679\\
                        0.54323251,0.57800947,0.10861208\\
                        0.64243828,0.71008907,0.05087847\\
                        0.74126349,0.84906277,0.02814695\\
                        0.83788445,0.94503478,0.01860762\\
                        0.923999,0.9896672,0.0128\\
                    } \platipusClassification
                    
                    \pgfplotstableread[col sep=comma, row sep=\\, header=true]{
                        confidence,accuracy,numSamples\\
                        0.27632109,0.33106001,0.0964829\\
                        0.35149913,0.35805536,0.3403526\\
                        0.44571405,0.45681017,0.2701978\\
                        0.54448435,0.56395793,0.15184483\\
                        0.64430216,0.68753562,0.08081434\\
                        0.74289311,0.82895935,0.04503932\\
                        0.840585,0.9329809,0.02890521\\
                        0.92871464,0.98394853,0.02130892\\
                    } \bmamlClassification
                    
                    \pgfplotstableread[col sep=comma, row sep=\\, header=true]{
                        confidence,accuracy,numSamples\\
                        0.27651079,0.33462842,0.07142643\\
                        0.35297438,0.33691995,0.27481682\\
                        0.44742597,0.42946424,0.25751935\\
                        0.54599329,0.51141695,0.16891899\\
                        0.64558311,0.61256428,0.10440475\\
                        0.74534512,0.70785884,0.06621873\\
                        0.84392902,0.81761952,0.04498409\\
                        0.93574736,0.92911248,0.03420713\\
                    } \abmlClassification
                    
                    \pgfplotstableread[col sep=comma, row sep=\\, header=true]{
                        confidence,accuracy,numSamples\\
                        0.27652626,0.33116672,0.06932022\\
                        0.35298069,0.35300478,0.26910389\\
                        0.44737684,0.43850623,0.25530874\\
                        0.54606441,0.52744696,0.17051084\\
                        0.6460898,0.61970942,0.10934366\\
                        0.74569204,0.72406059,0.07059324\\
                        0.84384474,0.83085577,0.04588402\\
                        0.93721819,0.93408788,0.03123835\\
                    } \simpaClassification
                    
                    \begin{axis}[
                        height = 0.9 \linewidth,
                        width = 0.9 \linewidth,
                        xlabel={Accuracy},
                        xlabel style={font=\footnotesize},
                        xticklabel style = {font=\footnotesize},
                        ylabel={\(\vert\) Confidence - Accuracy \(\vert\)},
                        ylabel style = {font=\footnotesize, yshift=0em},
                        yticklabel style = {font=\footnotesize, /pgf/number format/.cd, fixed, fixed zerofill, precision=2},
                        scale only axis,
                        legend entries = {MAML, PLATIPUS, BMAML, ABML, SImPa},
                        legend style = {draw=none, font=\scriptsize, yshift=0.25em, /tikz/every even column/.append style={column sep=0.5em}},
                        legend cell align = {left},
                        legend columns = 1,
                        legend image post style={scale=0.75},
                        legend pos = north west
                    ]
                        \addplot[mark=none, NavyBlue, dashdotted, very thick] table[x = {confidence}, y expr = abs(\thisrow{confidence} - \thisrow{accuracy})]{\mamlClassification};
                        \addplot[mark=none, PineGreen, dashed, very thick] table[x = {confidence}, y expr = abs(\thisrow{confidence} - \thisrow{accuracy})]{\platipusClassification};
                        \addplot[mark=none, BurntOrange, densely dashdotted, very thick] table[x = {confidence}, y expr = abs(\thisrow{confidence} - \thisrow{accuracy})]{\bmamlClassification};
                        \addplot[mark=none, RoyalPurple, dash pattern=on 4pt off 1pt, very thick] table[x = {confidence}, y expr = abs(\thisrow{confidence} - \thisrow{accuracy})]{\abmlClassification};
                        \addplot[mark=none, BrickRed, solid, very thick] table[x = {confidence}, y expr = abs(\thisrow{confidence} - \thisrow{accuracy})]{\simpaClassification};
                        
                        \addplot[mark=none, Black, densely dashed] table[y expr = \thisrow{confidence} - \thisrow{confidence}]{\mamlClassification};
                    \end{axis}

                \end{tikzpicture}
        		\caption{Reliability diagram}
        		\label{fig:SImPa_classification_reliability_chart}
        	\end{subfigure}
        	\hfill
        	\begin{subfigure}[b]{0.25\linewidth}
        		\centering
                \begin{tikzpicture}
                    \pgfplotstableread[col sep=comma, row sep=\\, header=true]{
                        method,ECE,MCE\\
                        MAML,0.040983405802980155,0.12373018632600719\\
                        PLATIPUS,0.03166020220231428,0.10779927760892039\\
                        BMAML,0.024685853492362714,0.09239589684429828\\
                        ABML,0.026369340992729282,0.05811763206558751\\
                        SImPa,0.014338726042537404,0.05464046\\
                    } \ceTable
                    
                    \pgfplotstablegetrowsof{\ceTable}
                    \pgfmathsetmacro{\NumRows}{\pgfplotsretval-1}

                    \begin{axis}[
                        width = 0.9 \linewidth,
                        height = 0.9 \linewidth,
                        xbar=0pt,
                        ytick=data,
                        yticklabels from table={\ceTable}{method},
                        yticklabel style = {font=\footnotesize},
                        xticklabel style = {font=\footnotesize},
                        xlabel style = {font=\footnotesize},
                        xlabel = {Calibration error},
                        scale only axis,
                        enlarge x limits=auto,
                        enlarge y limits=0.15,
                        grid=major,
                        grid style={dotted, thick},
                        every axis plot/.append style={fill, draw=none},
                        legend entries = {ECE,MCE},
                        legend style = {draw=none, font=\scriptsize, yshift=0em, /tikz/every even column/.append style={column sep=0.5em}},
                        legend cell align = {left},
                        legend columns = 2,
                        legend image post style={scale=1},
                        legend pos = south east
                    ]
                        \addplot[mark=none, NavyBlue] table [y expr=\NumRows - \coordindex, x = {ECE}] {\ceTable};
                        \addplot[mark=none, BrickRed] table [y expr=\NumRows - \coordindex, x = {MCE}] {\ceTable};
                    \end{axis}
                \end{tikzpicture}
        		\caption{ECE and MCE}
        		\label{fig:SImPa_classification_MCE_ECE}
        	\end{subfigure}
            \hfill
            \hphantom{dummy}

        	\caption{Calibration of the \protect\say{standard} 4-block CNN trained with different meta-learning methods on 5-way 1-shot classification tasks on mini-ImageNet.}
        	\label{fig:classification_calibration}
        \end{figure*}
        
        Similarly to the experiment for regression, we use reliability diagrams~\cite{guo2017oncalibration} to evaluate the predictive uncertainty. For a fair comparison, we re-implement several probabilistic meta-learning approaches, including MAML~\cite{finn2017model}, PLATIPUS~\cite{finn2018probabilistic}, BMAML~\cite{yoon2018bayesian} and ABML~\cite{ravi2018amortized}, using the 4-block CNN as the base model, trained under the same setting, and plot their reliability chart. The performance curves in the reliability diagram show how well calibrated a model is when testing across many unseen tasks. A perfectly calibrated model will have its values overlapped with the identity function \(y=x\), indicating that the probability associated with the label prediction is the same as the true probability. To ease the visualisation, we normalise the reliability chart by subtracting the predicted accuracy by its corresponding value on the diagonal \(y=x\), as shown in \figureautorefname~\ref{fig:SImPa_classification_reliability_chart}. Hence, for the normalised reliability chart, the closer to \(y=0\), the better the calibration. Visually, the model trained with SImPa shows better calibration than the ones trained with other meta-learning methods. To further evaluate, we compute the expected calibration error (ECE) and maximum calibration error (MCE)~\cite{guo2017oncalibration} of the models trained with these methods. The results plotted in \figureautorefname~\ref{fig:SImPa_classification_MCE_ECE} show that the model trained with SImPa achieves the smallest ECE and MCE among all the methods considered in this comparison.  The most competitive method to SImPa, regarding ECE and MCE, is ABML, but note that ABML has a worse classification accuracy than SImPa, as shown in \tableautorefname~\ref{tab:classification_accuracies} \emph{(Top)} -- see row \say{ABML~\cite{ravi2018amortized}}.

    \subsection{Discussion}
        \red{As mentioned in Remark~\ref{rm:bounded_loss}, the loss function \(\ell\) is clipped to \([0, 1]\) to satisfy the assumption made in PAC-Bayes framework. We observed that it affects the training of SImPa, mostly at the early stages. The reason might be the imbalance between the clipped loss and the regularisation terms related to KL divergence, making the learning over-regularised with slow convergence. In the implementation, we first train SImPa without such regularisation terms for 1,000 tasks, and then add them back to the loss with such regularisation. This facilitates the training process, allowing the algorithm to converge faster.}

        \red{As specified in \subsectionautorefname~\ref{sec:implicit_distribution}, the usage of implicit distribution, and in particular the generator in \eqref{eq:generator} to generate the parameter for the neural network of interest, leads to an exponential increasing of the number of learnable parameters. For example, in the classification of mini-ImageNet using the 4-convolutional module neural network, the number of parameter of the base network is 32,645, requiring us to have a generator with 16.7 million parameters. Such a large generator limits the scalability of SImPa, making it inapplicable for larger base networks. One workaround solution might be to utilise the architecture design proposed in Hypernetworks~\citep{ha2017hypernetworks} that shares parameters to reduce the size of the generator. The trade-off is the slight decreasing of the generator expressiveness when generating the parameter for the base neural network.}

        \red{Another limitation of SImPa is the need of GPU memory and training time comparing to other meta-learning approaches, such as MAML. In our implementation, the simplest baseline MAML needs 6 GPU-hours to train until convergence, while probabilistic baselines, such as ABML, BMAML and PLATIPUS, take about 30 GPU-hours to converge under the same setting. For SImPa, it requires more than 48 hours to converge. The reason of such long training time lies on the size of the meta-parameter, the need to train the \(\phi\) network to estimate KL divergence, and the number of Monte Carlo samples used. Please refer to \appendixautorefname~\ref{sec:complexity} for a detailed analysis of the running time complexity between SImPa and other meta-learning algorithms.}

\section{Conclusion}
	We introduce and formulate a new probabilistic algorithm for few-shot meta-learning. The proposed algorithm, SImPa, is based on PAC-Bayes framework which theoretically guarantees error induced on unseen tasks and unseen samples within each task. In addition, the proposed method employs a generative approach that implicitly models the posterior of task-specific model parameter \(q(\mathbf{w}_{i}; \lambda_{i})\), resulting in more expressive variational approximation compared to the usual variational methods using multivariate normal posteriors with diagonal covariance matrices, such as PLATIPUS~\cite{finn2018probabilistic} or ABML~\cite{ravi2018amortized}. The uncertainty, in the form of the learnt implicit distributions, can introduce more variability into the decision made by the model, resulting in well-calibrated and highly-accurate prediction. The algorithm can be combined with different base models that are trainable with gradient-based optimisation, and is applicable in regression and classification. We demonstrate that the algorithm has state-of-the-art calibration and prediction results on unseen data in a multi-modal 5-shot learning regression problem, and achieve state-of-the-art calibration and classification results on few-shot 5-way tasks on mini-ImageNet and tiered-ImageNet data sets.


%



\ifCLASSOPTIONcompsoc
  \section*{Acknowledgments}
\else
  \section*{Acknowledgment}
\fi

This work was supported by Australian Research Council grants CE140100016 and FT190100525. We also thank the Phoenix High Performance Computing service at the University of Adelaide to provide super-computing resources for this work.

\printbibliography

\ifCLASSOPTIONcaptionsoff
  \newpage
\fi

\onecolumn
\appendices
    \section{Proof of PAC-Bayes few-shot meta-learning bound}
\label{apdx:pac_bayes_proof}
    The derivation is divided into three steps. The first two steps are to derive the PAC-Bayes bound for the generalisation errors induced by the unseen tasks, and the unseen queried examples within each task. The novel bound is then constructed by combining the results obtained in the first two steps and presented in Theorem~\ref{theorem:pac_bayes_meta_learning}.

    \subsection{PAC-Bayes upper-bound for unseen validation of a single task}
    \label{sec:pac_bayes_bound_unseen_samples_single_task}
        This subsection presents a PAC-Bayes upper-bound for \(\mathbb{E}_{q(\theta; \psi)} \mathbb{E}_{\left( \mathcal{D}_{i}^{(v)}, f_{i} \right)} \mathbb{E}_{q(\mathbf{w}_{i}; \lambda_{i}^{*})} \left[ \ell \left(  \mathbf{x}_{ij}^{(v)}, y_{ij}^{(v)}; \mathbf{w}_{i} \right) \right]\) for task \(\mathcal{T}_{i}\). Note that this loss is different from the left hand-side term of Theorem~\ref{theorem:pac_bayes_single_task} (presented in \subsectionautorefname~\ref{sec:pac_bayes_single_task}) at the expectation over the posterior \(q(\theta; \psi)\).

        \begin{lemma}
        \label{lemma:pac_bayes_single_task}
            If \(S_{i}^{(v)}\) consists of \(m_{i}^{(v)}\) input \(\mathbf{x}_{ik}^{(v)}\) sampled from a data probability distribution \(\mathcal{D}_{i}\) and labelled by \(f_{i}\), \(\mathcal{H}\) is a hypothesis class with each hypothesis \(h\) parameterised by \(\mathbf{w}_{i}\), \(\ell: \mathcal{H} \times \mathcal{Y} \to [0, 1]\) is a loss function, \(q(\mathbf{w}_{i}; \lambda_{i}^{*})\) is a \say{posterior} over the hypothesis parameter and \(p(\mathbf{w}_{i})\) is a prior, then the following holds:
            \begin{equation*}
                \begin{aligned}[b]
                    \mathrm{Pr} \left( \mathbb{E}_{q(\theta; \psi)} \mathbb{E}_{\left( \mathcal{D}_{i}^{(v)}, f_{i} \right)} \mathbb{E}_{q(\mathbf{w}_{i}; \lambda_{i}^{*})} \left[ \ell \left(  \mathbf{x}_{ij}^{(v)}, y_{ij}^{(v)}; \mathbf{w}_{i} \right) \right] \le \frac{1}{m_{i}^{(v)}} \sum_{k=1}^{m_{i}^{(v)}} \mathbb{E}_{q(\theta; \psi)} \mathbb{E}_{q(\mathbf{w}_{i}; \lambda_{i}^{*})} \left[ \ell \left(  \mathbf{x}_{ik}^{(v)}, y_{ik}^{(v)}; \mathbf{w}_{i} \right) \right] + R_{i} \right) \ge 1 - \varepsilon_{i},
                \end{aligned}
            \end{equation*}
            where: \(\varepsilon_{i} \in (0, 1]\), and
            \begin{equation}
                R_{i} = \sqrt{\frac{\mathbb{E}_{q(\theta; \psi)} \left[ \mathrm{KL} \left[ q(\mathbf{w}_{i}; \lambda_{i}^{*}) || p(\mathbf{w}_{i}) \right] \right] + \frac{\ln m_{i}^{(v)}}{\varepsilon_{i}}}{2(m_{i}^{(v)} - 1}}
                \label{eq:ri}
            \end{equation}
        \end{lemma}
        
        \begin{proof}
            To simplify the proof, let \(\Delta \mathsf{L}\) is the difference between the true loss and empirical loss:
            \begin{equation*}
                \Delta \mathsf{L} = \mathbb{E}_{\left( \mathcal{D}_{i}^{(v)}, f_{i} \right)} \left[ \ell \left(  \mathbf{x}_{ij}^{(v)}, y_{ij}^{(v)}; \mathbf{w}_{i} \right) \right] - \frac{1}{m_{i}^{(v)}} \sum_{k=1}^{m_{i}^{(v)}} \ell \left(  \mathbf{x}_{ik}^{(v)}, y_{ik}^{(v)}; \mathbf{w}_{i} \right).
            \end{equation*}
            Given the Fubini's theorem to interchange the expectations, the problem can be written as:
            \begin{equation}
                \mathrm{Pr} \left( \mathbb{E}_{q(\theta; \psi)} \mathbb{E}_{q(\mathbf{w}_{i}; \lambda_{i}^{*})} \left[ \Delta \mathsf{L} \right] \le \sqrt{\frac{\mathbb{E}_{q(\theta; \psi)} \left[ \mathrm{KL} \left[ q(\mathbf{w}_{i}; \lambda_{i}^{*}) || p(\mathbf{w}_{i}) \right] \right] + \frac{\ln m_{i}^{(v)}}{\varepsilon_{i}}}{2(m_{i}^{(v)} - 1}} \, \right) \ge 1 - \varepsilon_{i}.
            \end{equation}

            Applying Lemma~\ref{lmm:compression} (presented in \appendixautorefname~\ref{apdx:auxiliary_lemmas}) with \(2 (m_{i}^{(v)} - 1) \Delta \mathsf{L}^{2}\) as \(\phi(h)\) gives:
            \begin{equation}
                2 \left( m_{i}^{(v)} - 1 \right) \mathbb{E}_{q(\mathbf{w}_{i}; \lambda_{i}^{*})} \left[ \mathsf{L}^{2} \right] \le \mathrm{KL} \left[ q(\mathbf{w}_{i}; \lambda_{i}^{*}) || p(\mathbf{w}_{i}) \right] + \ln \mathbb{E}_{p(\mathbf{w}_{i})} \left[ e^{ 2 \left( m_{i}^{(v)} - 1 \right) \Delta \mathsf{L}^{2} } \right].
                \label{eq:corolary_compression_lemma}
            \end{equation}
            
            To prove Lemma~\ref{lemma:pac_bayes_single_task}, the left-hand side term in \eqref{eq:corolary_compression_lemma} is lower-bounded, while the last term of the right-hand side in \eqref{eq:corolary_compression_lemma} is upper-bounded.
            
            To lower-bound the left-hand side term, Jensen's inequality is applied on the convex function \(x^{2}\) to obtain:
            \begin{equation}
                2 \left( m_{i}^{(v)} - 1 \right) \left( \mathbb{E}_{q(\mathbf{w}_{i}; \lambda_{i}^{*})} \left[ \Delta \mathsf{L} \right] \right)^{2} \le 2 \left( m_{i}^{(v)} - 1 \right) \mathbb{E}_{q(\mathbf{w}_{i}; \lambda_{i}^{*})} \left[ \mathsf{L}^{2} \right].
                \label{eq:jensen_quadratic}
            \end{equation}
            The results in \eqref{eq:corolary_compression_lemma} and \eqref{eq:jensen_quadratic} lead to:
            \begin{equation}
                \sqrt{2 \left( m_{i}^{(v)} - 1 \right)} \mathbb{E}_{q(\mathbf{w}_{i}; \lambda_{i}^{*})} \left[ \Delta \mathsf{L} \right] \le \sqrt{ \mathrm{KL} \left[ q(\mathbf{w}_{i}; \lambda_{i}^{*}) || p(\mathbf{w}_{i}) \right] + \ln \mathbb{E}_{p(\mathbf{w}_{i})} \left[ e^{ 2 \left( m_{i}^{(v)} - 1 \right) \Delta \mathsf{L}^{2} } \right] }.
            \end{equation}
            
            Taking expectation over \(q(\theta; \psi)\) on both sides gives:
            \begin{equation}
                \sqrt{2 \left( m_{i}^{(v)} - 1 \right)} \mathbb{E}_{q(\theta; \psi)} \mathbb{E}_{q(\mathbf{w}_{i}; \lambda_{i}^{*})} \left[ \Delta \mathsf{L} \right] \le \mathbb{E}_{q(\theta; \psi)} \left[ \, \sqrt{ \mathrm{KL} \left[ q(\mathbf{w}_{i}; \lambda_{i}^{*}) || p(\mathbf{w}_{i}) \right] + \ln \mathbb{E}_{p(\mathbf{w}_{i})} \left[ e^{ 2 \left( m_{i}^{(v)} - 1 \right) \Delta \mathsf{L}^{2} } \right] } \, \right].
            \end{equation}

            Note that \(\sqrt{x}\) is a concave function. Hence, one can apply Jensen's inequality on the right-hand side term to obtain an upper-bound. This results in:
            \begin{equation}
                \sqrt{ 2 \left( m_{i}^{(v)} - 1 \right) } \mathbb{E}_{q(\theta; \psi)} \mathbb{E}_{q(\mathbf{w}_{i}; \lambda_{i}^{*})} \left[ \Delta \mathsf{L} \right] \le \sqrt{ \mathbb{E}_{q(\theta; \psi)} \left[ \mathrm{KL} \left[ q(\mathbf{w}_{i}; \lambda_{i}^{*}) || p(\mathbf{w}_{i}) \right] \right] + \mathbb{E}_{q(\theta; \psi)} \left[ \ln \mathbb{E}_{p(\mathbf{w}_{i})} \left[ e^{ 2 \left( m_{i}^{(v)} - 1 \right) \Delta \mathsf{L}^{2} } \right] \right] }.
                \label{eq:lower_bound_lhs}
            \end{equation}
            
            To upper-bound the last term in \eqref{eq:corolary_compression_lemma}, Lemma~\ref{lmm:exercise311} (presented in \appendixautorefname~\ref{apdx:auxiliary_lemmas}) is then used. To do that, \(\Delta \mathsf{L}\) is required to satisfied the assumption of Lemma~\ref{lmm:exercise311}. This can be done by applying the Hoeffding's inequality.

            Consider each loss value \(\ell(\mathbf{x}_{ik}^{(v)}, y_{ik}; \mathbf{w}_{i}) \in [0, 1]\) as an i.i.d. random variable with true mean \(\mathbb{E}_{\left( \mathcal{D}_{i}^{(v)}, f_{i} \right)} \left[ \ell \left(  \mathbf{x}_{ij}^{(v)}, y_{ij}^{(v)}; \mathbf{w}_{i} \right) \right]\) and empirical mean \(\frac{1}{m_{i}^{(v)}} \sum_{k=1}^{m_{i}^{(v)}} \ell \left(  \mathbf{x}_{ik}^{(v)}, y_{ik}^{(v)}; \mathbf{w}_{i} \right)\). Hence, applying Hoeffding's inequality gives:
            \begin{equation}
                \mathrm{Pr} \left( \left| \Delta L \right| \ge \epsilon \right) \le e^{-2 m_{i}^{(v)} \epsilon^{2}}, \forall \epsilon \ge 0
            \end{equation}
            
            Therefore, this allows to apply Lemma~\ref{lmm:exercise311} to upper-bound the last term in the right-hand side of \eqref{eq:corolary_compression_lemma}:
            \begin{equation}
                \mathbb{E}_{(\mathcal{D}_{i}^{m_{i}^{(v)}}, f)} \left[ e^{ 2 \left( m_{i}^{(v)} - 1 \right) \Delta \mathsf{L}^{2} } \right] \le m_{i}^{(v)}.
            \end{equation}
            
            Taking the expectation w.r.t. the distribution \(q(\mathbf{w}_{i}; \lambda_{i}^{*})\)on both sides gives:
            \begin{equation}
                \mathbb{E}_{(\mathcal{D}_{i}^{m_{i}^{(v)}}, f)} \mathbb{E}_{q(\mathbf{w}_{i}; \lambda_{i}^{*})} \left[ e^{ 2 \left( m_{i}^{(v)} - 1 \right) \Delta \mathsf{L}^{2} } \right] \le m_{i}^{(v)}.
            \end{equation}
            Note that the two expectations on the left-hand side term are interchanged due to Fubini's theorem.
            
            Taking the logarithm on both side, and applying Jensen's inequality to lower-bound the left-hand side term give:
            \begin{equation}
                \mathbb{E}_{(\mathcal{D}_{i}^{m_{i}^{(v)}}, f)} \ln \mathbb{E}_{q(\mathbf{w}_{i}; \lambda_{i}^{*})} \left[ e^{ 2 \left( m_{i}^{(v)} - 1 \right) \Delta \mathsf{L}^{2} } \right] \le \ln \mathbb{E}_{(\mathcal{D}_{i}^{m_{i}^{(v)}}, f)} \mathbb{E}_{q(\mathbf{w}_{i}; \lambda_{i}^{*})} \left[ e^{ 2 \left( m_{i}^{(v)} - 1 \right) \Delta \mathsf{L}^{2} } \right] \le \ln m_{i}^{(v)}.
            \end{equation}
            Taking the expectation over the distribution \(q(\theta; \psi)\) on both sides and applying Fubini's theorem to interchange the two expectations on the left-hand side give:
            \begin{equation}
                \mathbb{E}_{(\mathcal{D}_{i}^{m_{i}^{(v)}}, f)} \mathbb{E}_{q(\theta; \psi)} \ln \mathbb{E}_{q(\mathbf{w}_{i}; \lambda_{i}^{*})} \left[ e^{ 2 \left( m_{i}^{(v)} - 1 \right) \Delta \mathsf{L}^{2} } \right] \le \ln m_{i}^{(v)}.
            \end{equation}
            The lower-bound (or the term at the left-hand side of the above inequality) can be lower-bounded further by applying Markov's inequality:
            \begin{equation}
                \mathrm{Pr} \left( \mathbb{E}_{q(\theta; \psi)} \ln \mathbb{E}_{q(\mathbf{w}_{i}; \lambda_{i}^{*})} \left[ e^{ 2 \left( m_{i}^{(v)} - 1 \right) \Delta \mathsf{L}^{2} } \right] \ge \epsilon \right) \le \frac{ \mathbb{E}_{(\mathcal{D}_{i}^{m_{i}^{(v)}}, f)} \mathbb{E}_{q(\theta; \psi)} \ln \mathbb{E}_{q(\mathbf{w}_{i}; \lambda_{i}^{*})} \left[ e^{ 2 \left( m_{i}^{(v)} - 1 \right) \Delta \mathsf{L}^{2} } \right] }{ \epsilon }, \forall \epsilon > 0.
            \end{equation}

            This implies that:
            \begin{equation}
                \mathrm{Pr} \left( \mathbb{E}_{q(\theta; \psi)} \ln \mathbb{E}_{q(\mathbf{w}_{i}; \lambda_{i}^{*})} \left[ e^{ 2 \left( m_{i}^{(v)} - 1 \right) \Delta \mathsf{L}^{2} } \right] \ge \epsilon \right) \le \frac{ \ln m_{i}^{(v)} }{ \epsilon }, \forall \epsilon > 0.
            \end{equation}
            
            Equivalently, one can write the above inequality as:
            \begin{equation}
                \mathrm{Pr} \left( \mathbb{E}_{q(\theta; \psi)} \ln \mathbb{E}_{q(\mathbf{w}_{i}; \lambda_{i}^{*})} \left[ e^{ 2 \left( m_{i}^{(v)} - 1 \right) \Delta \mathsf{L}^{2} } \right] \le \epsilon \right) \ge 1 - \frac{ \ln m_{i}^{(v)} }{ \epsilon }, \forall \epsilon > 0.
            \end{equation}
            
            Hence, adding an expectation of a KL divergence on both sides of the inequality inside the probability function and taking square root gives:
            \begin{equation}
                \begin{aligned}[b]
                    & \mathrm{Pr} \left( \sqrt{ \mathbb{E}_{q(\theta; \psi)} \left[ \mathrm{KL} \left[ q(\mathbf{w}_{i}; \lambda_{i}^{*}) || p(\mathbf{w}_{i}) \right] \right] + \mathbb{E}_{q(\theta; \psi)} \left[ \ln \mathbb{E}_{p(\mathbf{w}_{i})} \left[ e^{ 2 \left( m_{i}^{(v)} - 1 \right) \Delta \mathsf{L}^{2} } \right] \right] } \le \right.\\
                    & \qquad \left. \le \sqrt{ \mathbb{E}_{q(\theta; \psi)} \left[ \mathrm{KL} \left[ q(\mathbf{w}_{i}; \lambda_{i}^{*}) || p(\mathbf{w}_{i}) \right] \right] + \epsilon } \, \vphantom{\frac{1}{2}} \right) \ge 1 - \frac{\ln m_{i}^{(v)}}{\epsilon}, \forall \epsilon > 0.
                \end{aligned}
                \label{eq:lemma2_28}
            \end{equation}
            The results in \eqref{eq:lower_bound_lhs} and \eqref{eq:lemma2_28} lead to:
            \begin{equation}
                \mathrm{Pr} \left( \sqrt{ 2 \left( m_{i}^{(v)} - 1 \right) } \mathbb{E}_{q(\theta; \psi)} \mathbb{E}_{q(\mathbf{w}_{i}; \lambda_{i}^{*})} \left[ \Delta \mathsf{L} \right] \le \sqrt{ \mathbb{E}_{q(\theta; \psi)} \left[ \mathrm{KL} \left[ q(\mathbf{w}_{i}; \lambda_{i}^{*}) || p(\mathbf{w}_{i}) \right] \right] + \epsilon } \right) \ge 1 - \frac{\ln m_{i}^{(v)}}{\epsilon}, \forall \epsilon > 0.
            \end{equation}
            
            Setting \(\varepsilon_{i} = \frac{\ln m_{i}^{(v)}}{\epsilon}\) and dividing both sides of the inequality inside the probability function by \(\sqrt{ 2 \left( m_{i}^{(v)} - 1 \right) }\) give:
            \begin{equation}
                \mathrm{Pr} \left( \mathbb{E}_{q(\theta; \psi)} \mathbb{E}_{q(\mathbf{w}_{i}; \lambda_{i}^{*})} \left[ \Delta \mathsf{L} \right] \le \sqrt{ \frac{ \mathbb{E}_{q(\theta; \psi)} \left[ \mathrm{KL} \left[ q(\mathbf{w}_{i}; \lambda_{i}^{*}) || p(\mathbf{w}_{i}) \right] \right] + \frac{\ln m_{i}^{(v)}}{\varepsilon_{i}} }{2 \left( m_{i}^{(v)} - 1 \right)} }\right) \ge 1 - \varepsilon_{i}.
            \end{equation}
        \end{proof}

    \subsection{PAC-Bayes upper-bound for unseen tasks}
    \label{sec:pac_bayes_bound_unseen_tasks}
        The PAC-Bayes upper-bound on the generalisation loss for unseen tasks can be obtained as a corollary of Theorem~\ref{theorem:pac_bayes_single_task} (presented in \subsectionautorefname~\ref{sec:pac_bayes_single_task}). In particular, if:
        \begin{itemize}
            \item the loss function is defined as \(\mathbb{E}_{q(\mathbf{w}_{i}; \lambda_{i}^{*})} \mathbb{E}_{\left( \mathcal{D}_{i}^{(v)}, f_{i} \right)} \left[ \ell \left(  \mathbf{x}_{ij}^{(v)}, y_{ij}^{(v)}; \mathbf{w}_{i} \right) \right]\),
            \item data generation is the task environment \(p(\mathcal{D}, f)\),
            \item dataset consists of \(T > 1\) tasks i.i.d. sampled from the task environment,
            \item hypothesis is \(\theta\),
            \item posterior is \(q(\theta; \psi)\),
            \item and prior is \(p(\theta)\),
        \end{itemize}
        then one can apply Theorem~\ref{theorem:pac_bayes_single_task} to obtain a PAC-Bayes upper-bound for unseen tasks as shown in Corollary~\ref{corollary:pac_bayes_unseen_tasks}.
        
        \begin{corollary}
        \label{corollary:pac_bayes_unseen_tasks}
            \begin{equation*}
                \begin{aligned}[b]
                    & \mathrm{Pr} \left( \mathbb{E}_{q(\theta; \psi)} \mathbb{E}_{p(\mathcal{D}, f)} \mathbb{E}_{q(\mathbf{w}_{i}; \lambda_{i}^{*})} \mathbb{E}_{\left( \mathcal{D}_{i}^{(v)}, f_{i} \right)} \left[ \ell \left(  \mathbf{x}_{ij}^{(v)}, y_{ij}^{(v)}; \mathbf{w}_{i} \right) \right] \le \right.\\
                    & \qquad \left. \le \frac{1}{T} \sum_{i = 1}^{T} \mathbb{E}_{q(\theta; \psi)} \mathbb{E}_{q(\mathbf{w}_{i}; \lambda_{i}^{*})} \mathbb{E}_{\left( \mathcal{D}_{i}^{(v)}, f_{i} \right)} \left[ \ell \left(  \mathbf{x}_{ij}^{(v)}, y_{ij}^{(v)}; \mathbf{w}_{i} \right) \right] + R_{0} \right) \ge 1 - \varepsilon_{0},
                \end{aligned}
            \end{equation*}
            where: \(\varepsilon_{0} \in (0, 1]\), and
            \begin{equation}
                R_{0} = \sqrt{ \frac{ \mathrm{KL} \left[ q(\theta; \psi) || p(\theta) \right] + \frac{\ln T}{\varepsilon_{0}} }{2(T - 1)} }.
                \label{eq:r0}
            \end{equation}
        \end{corollary}

    \subsection{PAC-Bayes upper-bound for meta-learning}
    \label{sec:pac_bayes_meta_learning}
        \MetaBound*
        \begin{proof}
            First, the upper-bound for the unseen examples of a single-task obtained from Lemma~\ref{lemma:pac_bayes_single_task} (presented in \appendixautorefname~\ref{sec:pac_bayes_bound_unseen_samples_single_task}) is extended for \(T\) training tasks by using Lemma~\ref{lemma:union_bound} (presented in \appendixautorefname~\ref{apdx:auxiliary_lemmas}) with the following substitution:
            \begin{itemize}
                \item \(X_{i} := \mathbb{E}_{q(\theta; \psi)} \mathbb{E}_{\left( \mathcal{D}_{i}^{(v)}, f_{i} \right)} \mathbb{E}_{q(\mathbf{w}_{i}; \lambda_{i}^{*})} \left[ \ell \left(  \mathbf{x}_{ij}^{(v)}, y_{ij}^{(v)}; \mathbf{w}_{i} \right) \right]\)
                \item \(Y_{i} := \frac{1}{m_{i}^{(v)}} \sum_{k=1}^{m_{i}^{(v)}} \mathbb{E}_{q(\theta; \psi)} \mathbb{E}_{q(\mathbf{w}_{i}; \lambda_{i}^{*})} \left[ \ell \left(  \mathbf{x}_{ik}^{(v)}, y_{ik}^{(v)}; \mathbf{w}_{i} \right) \right] + R_{i}\)
            \end{itemize}
            to obtain:
            \begin{equation}
                \begin{aligned}[b]
                    & \mathrm{Pr} \left( \frac{1}{T} \sum_{i = 1}^{T} \mathbb{E}_{q(\theta; \psi)} \mathbb{E}_{\left( \mathcal{D}_{i}^{(v)}, f_{i} \right)} \mathbb{E}_{q(\mathbf{w}_{i}; \lambda_{i}^{*})} \left[ \ell \left(  \mathbf{x}_{ij}^{(v)}, y_{ij}^{(v)}; \mathbf{w}_{i} \right) \right] \le \right. \\
                    & \qquad \left. \le \frac{1}{T} \sum_{i = 1}^{T} \frac{1}{m_{i}^{(v)}} \sum_{k=1}^{m_{i}^{(v)}} \mathbb{E}_{q(\theta; \psi)} \mathbb{E}_{q(\mathbf{w}_{i}; \lambda_{i}^{*})} \left[ \ell \left(  \mathbf{x}_{ik}^{(v)}, y_{ik}^{(v)}; \mathbf{w}_{i} \right) \right] + R_{i} \right) \ge 1 - \sum_{i=1}^{T} \varepsilon_{i}.
                \end{aligned}
                \label{eq:pac_bayes_many_tasks}
            \end{equation}

            Given Corollary~\ref{corollary:pac_bayes_unseen_tasks} (presented in \appendixautorefname~\ref{sec:pac_bayes_bound_unseen_tasks}) and the result in \eqref{eq:pac_bayes_many_tasks}, one can apply Corollary~\ref{crll:commutative_bound} (presented in \appendixautorefname~\ref{apdx:auxiliary_lemmas}) to obtain the following:
            \begin{equation}
                \begin{aligned}[b]
                    & \mathrm{Pr} \left( \mathbb{E}_{q(\theta; \psi)} \mathbb{E}_{(\mathcal{D}, f)} \mathbb{E}_{\left( \mathcal{D}_{i}^{(v)}, f_{i} \right)} \mathbb{E}_{q(\mathbf{w}_{i}; \lambda_{i}^{*})} \left[ \ell \left(  \mathbf{x}_{ij}^{(v)}, y_{ij}^{(v)}; \mathbf{w}_{i} \right) \right] \le \right.\\
                    & \qquad \left. \le \frac{1}{T} \sum_{i = 1}^{T} \frac{1}{m_{i}^{(v)}} \sum_{k=1}^{m_{i}^{(v)}} \mathbb{E}_{q(\theta; \psi)} \mathbb{E}_{q(\mathbf{w}_{i}; \lambda_{i}^{*})} \left[ \ell \left(  \mathbf{x}_{ik}^{(v)}, y_{ik}^{(v)}; \mathbf{w}_{i} \right) \right] + R_{i} + R_{0} \right) \ge 1 - \sum_{j=0}^{T} \varepsilon_{j}.
                \end{aligned}
            \end{equation}
            Setting \(\varepsilon_{0} = \frac{\varepsilon}{T}\) and \(\varepsilon_{i} = \frac{(T - 1) \varepsilon}{T^{2}}, \forall i \in \{1, \ldots, T\}\) completes the proof.
        \end{proof}

    \newpage
    \section{Auxiliary lemmas}
\label{apdx:auxiliary_lemmas}
    
    \KLLowerBound*
    \begin{proof}
        For a measurable function \(\phi(h)\):
        \begin{equation}
            \mathbb{E}_{Q} \left[ \phi(h) \right] = \mathbb{E}_{Q}\left[ \ln \left( e^{\phi(h)} \frac{Q(h)}{P(h)} \frac{P(h)}{Q(h)} \right) \right] = \mathrm{KL} \left[ Q || P \right] + \mathbb{E}_{Q} \left[ \ln \left( e^{\phi(h)} \frac{P(h)}{Q(h)} \right) \right].
        \end{equation}
        Applying Jensen's inequality on the last term gives:
        \begin{equation}
            \mathbb{E}_{Q} \left[ \phi(h) \right] \le \mathrm{KL} \left[ Q || P \right] + \ln \mathbb{E}_{Q} \left[ e^{\phi(h)} \frac{P(h)}{Q(h)} \right] = \mathrm{KL} \left[ Q || P \right] + \ln \mathbb{E}_{P} \left[ e^{\phi(h)} \right].
        \end{equation}
        Re-arrange the term proves the first part of the lemma, which is the lower-bound of the KL divergence.

        To prove the second part of the lemma, we simply show that there exists a function \(\phi(h)\) that makes the lower-bound achieve the maximal value which is the KL divergence. Let:
        \begin{equation}
            \phi(h) = \ln \frac{Q(h)}{P(h)},
        \end{equation}
        then the lower-bound can be expressed as:
        \begin{equation}
            \begin{aligned}[b]
                \mathbb{E}_{Q} \left[ \phi(h) \right] - \ln \mathbb{E}_{P} \left[ e^{\phi(h)} \right] & = \mathrm{KL} \left[ Q || P \right] - \ln \mathbb{E}_{P} \left[ \frac{Q(h)}{P(h)} \right] = \mathrm{KL} \left[ Q || P \right].
            \end{aligned}
        \end{equation}
        That completes the proof.
    \end{proof}
    
    \begin{lemma}[Exercise 31.1 in \citep{shalev2014understanding}]
    \label{lmm:exercise311}
        Let \(X\) be a non-negative random variable that satisfies: \(\mathrm{Pr} \left( X \ge \epsilon \right) \le e^{-2m\epsilon^{2}}, \forall \epsilon \ge 0\). Prove that: \(\mathbb{E} \left[ e^{2(m - 1) X^{2}} \right] \le m\).
    \end{lemma}
    
    \begin{proof}
        We will present the expectation of interest in term of probability of \(X\). For simplicity, let \(Y = e^{2(m - 1) X^{2}}\). Since \(X^{2} \in [0, +\infty)\), then \(Y \in [1, +\infty)\). According to the layer cake representation~\citep[Page 26]{lieb2001analysis}:
        \begin{equation*}
            Y = \int_{0}^{Y} dt = \int_{1}^{+ \infty} \mathbbm{1} \left( Y \ge t \right) dt,
        \end{equation*}
        where \(\mathbbm{1}(A)\) is the indicator function of event \(A\).
        
        One important property of indicator function is that:
        \begin{equation*}
            \mathbb{E} \left[ \mathbbm{1} \left( Y \ge t \right) \right] = \mathrm{Pr} \left( Y \ge t \right).
        \end{equation*}.
        
        With the above representation, we can express the expectation of interest as:
        \begin{equation*}
            \begin{aligned}[b]
                \mathbb{E} \left[ Y \right] & = \mathbb{E} \left[ \int_{1}^{+ \infty} \mathbbm{1} \left( Y \ge t \right) dt \right] \\
                & = \int_{1}^{+ \infty} \mathbb{E} \left[\mathbbm{1} \left( Y \ge t \right) \right] dt \quad \text{(Fubini's theorem)}\\
                & = \int_{1}^{+\infty} \mathrm{Pr} \left( Y \ge t \right) dt.
            \end{aligned}
        \end{equation*}
        Or:
        \begin{equation*}
            \mathbb{E} \left[ e^{2(m - 1) X^{2}} \right] = \int_{1}^{+ \infty} \mathrm{Pr} \left( e^{2(m - 1) X^{2}} \ge x \right) dx.
        \end{equation*}
        We will change the variable from \(x\) to \(\epsilon\) to utilise the given inequality. Let:
        \begin{equation*}
            x = e^{2(m - 1) \epsilon^{2}},
        \end{equation*}
        and since \(\epsilon \ge 0\), then:
        \begin{equation*}
            \epsilon = \sqrt{\frac{\ln x}{2(m - 1)}},
        \end{equation*}
        and
        \begin{equation*}
            dx = 4(m - 1) \epsilon e^{2(m - 1) \epsilon^{2}} d\epsilon.
        \end{equation*}
        Hence, the expectation of interest can be written in term of the changed variable \(\epsilon\) as:
        \begin{equation*}
            \begin{aligned}[b]
                \mathbb{E} \left[e^{2(m - 1) X^{2}} \right] & = \int_{0}^{+\infty} \mathrm{Pr} \left( e^{2(m - 1) X^{2}} \ge e^{2(m - 1) \epsilon^{2}} \right) 4(m - 1) \epsilon e^{2(m - 1) \epsilon^{2}} d\epsilon \\
                & = \int_{0}^{+\infty} \underbrace{\mathrm{Pr} \left( X \ge \epsilon \right)}_{\le e^{-2m\epsilon^{2}}} 4(m - 1) \epsilon e^{2(m - 1) \epsilon^{2}} d\epsilon \\
                & \le \int_{0}^{+\infty} 4(m - 1) \epsilon e^{-2 \epsilon^{2}} d\epsilon = m - 1 < m.
            \end{aligned}
        \end{equation*}
    \end{proof}

    \begin{lemma}
    \label{lmm:probability_inequality_sum}
        For \(i=1:n\), if \(X_{i}\) and \(Y_{i}\) are random variables, then:
        \begin{equation*}
            p \left( \sum_{i=1}^{n} X_{i} \le \sum_{i=1}^{n} Y_{i} \right) \ge p \left( \bigcap_{i=1}^{n} \left( X_{i} \le Y_{i} \right) \right).
        \end{equation*}
    \end{lemma}
    
    \begin{proof}
        The proof is quite direct:
        \begin{equation}
            X_{i} \le Y_{i} \implies \sum_{i=1}^{n} X_{i} \le \sum_{i=1}^{n} Y_{i}.
        \end{equation}
        Hence, applying the probability for implication completes the proof.
    \end{proof}
    
    \begin{lemma}
    \label{lmm:inequality_intersection_n_events}
        For \(n\) events \(A_{i}\) with \(i=1:n\), the following holds:
        \begin{equation*}
            p \left( \bigcap_{i=1}^{n} A_{i} \right) \ge \left( \sum_{i=1}^{n} p(A_{i}) \right) - (n - 1), \, \forall n \ge 2.
        \end{equation*}
    \end{lemma}
    
    \begin{proof}
        Proof can be done by induction.
        
        For \(n = 2\):
        \begin{equation*}
            \begin{aligned}[b]
                p(A_{1} \cap A_{2}) = p(A_1) + p(A_2) - p(A_{1} \cup A_{2}) \ge p(A_{1}) + p(A_{2}) - 1.
            \end{aligned}
        \end{equation*}
        
        Suppose that it is true for case \(n\):
        \begin{equation*}
            p \left( \bigcap_{i=1}^{n} A_{i} \right) \ge \left( \sum_{i=1}^{n} p(A_{i}) \right) - (n - 1).
        \end{equation*}
        We prove that this is also true for case \((n + 1)\):
        \begin{equation*}
            \begin{aligned}
                p \left( \bigcap_{i=1}^{n+1} A_{i} \right) & = p \left( \bigcap_{i=1}^{n} A_{i} \right) + p(A_{n+1}) - p \left( \left(\bigcap_{i=1}^{n} A_{i} \right) \bigcup A_{n+1} \right) \\
                & \ge p \left( \bigcap_{i=1}^{n} A_{i} \right) + p(A_{n+1}) - 1 \\
                & \ge \left( \sum_{i=1}^{n} p(A_{i}) \right) - (n - 1) + p(A_{n+1}) - 1 \\
                & \quad\text{(assumption of induction for case \(n\))}\\
                & \ge \left( \sum_{i=1}^{n+1} p(A_{i}) \right) - \left((n + 1) - 1 \right).
            \end{aligned}
        \end{equation*}
        It is, therefore, true for \((n+1)\), and hence, the proof.
    \end{proof}
    
    \begin{lemma}
    \label{lemma:union_bound}
        Let \(X_{i}\) and \(Y_{i}\) are random variables with \(i=1:n\). If \(p(X_{i} \le Y_{i}) \ge 1 - \delta_{i}\) with \(\delta_{i} \in (0, 1]\), then:
        \begin{equation*}
            p \left( \sum_{i=1}^{n} X_{i} \le \sum_{i=1}^{n} Y_{i} \right) \ge 1 - \sum_{i=1}^{n} \delta_{i}.
        \end{equation*}
    \end{lemma}
    
    \begin{proof}
        Applying Lemmas~\ref{lmm:probability_inequality_sum} and \ref{lmm:inequality_intersection_n_events} for the left-hand side term of this lemma gives:
        \begin{equation}
            \begin{aligned}[b]
                p \left( \sum_{i=1}^{n} X_{i} \le \sum_{i=1}^{n} Y_{i} \right) & \ge p \left( \bigcap_{i=1}^{n} \left( X_{i} \le Y_{i} \right) \right) \quad (\text{Lemma~\ref{lmm:probability_inequality_sum}}) \\
                & \ge \sum_{i=1}^{n} p \left( \left( X_{i} \le Y_{i} \right) \right) - (n - 1) \quad (\text{Lemma~\ref{lmm:inequality_intersection_n_events}}) \\
                & \ge \sum_{i=1}^{n} \left(1 - \delta_{i}\right) - (n - 1) \\
                & = 1 - \sum_{i=1}^{n} \delta_{i}.
            \end{aligned}
        \end{equation}
    \end{proof}
    
    \begin{corollary}
    \label{crll:commutative_bound}
        If \(p(a \le b) \ge 1 - \delta_{1}\) and \(p(b \le c) \ge 1 - \delta_{2}\) with \(\delta_{1}, \delta_{2} \in (0, 1]\), then:
        \begin{equation*}
            p(a \le c) \ge 1 - \delta_{1} - \delta_{2}.
        \end{equation*}
    \end{corollary}

    \newpage
    \section{Complexity analysis}
\label{sec:complexity}
    In this \sectionautorefname, we analyse the running time complexity of different meta-learning algorithms related to SImPa per one gradient update for the parameter of interest. These methods include:
    \begin{itemize}
        \item point estimate such as MAML~\citep{finn2017model},
        \item probabilistic modelling based on variational inference where the posterior is approximated by a multivariate normal distribution with a diagonal covariance matrix~\citep{ravi2018amortized, nguyen2020uncertainty},
        \item SImPa.
    \end{itemize}
    For simplicity, we assume that the number of samples within each training subset \(S_{i}^{(t)}\) is the same across all tasks: \(m_{0}^{(t)} = m_{i}^{(t)}, \forall i \in \{1, \ldots, T\}\), and so is the validation subset. In addition, as all the methods mentioned are implemented with their first-order versions, the analysis also relies on such assumption. Furthermore, given that such algorithms are implemented with automatic differentiation, the running time complexity of a back-propagation is linear w.r.t. the number of the model's parameters.

    To ease the analysis, we re-state the definition of some notations and define some new ones as shown in \tableautorefname~\ref{tab:notations}.

    \begin{table}[h]
        \centering
        \begin{tabular}{l l}
            \toprule
            \bfseries Notations & \bfseries Description \\
            \midrule
            \(n\) & the number of base model parameters \(= |\mathbf{w}_{i}|\)\\
            \(n^{\prime}\) & the number of the hyper-parameters \( = |\bm{\mu}_{\theta}| = |\theta|\) \\
            \(N_{\mathrm{lower}}\) & the number of gradient updates to minimise lower-level function in \eqref{eq:meta_learning_objective} \\
            \(N_{\phi}\) & the number of gradient updates to learn the task-specific \(\phi\)-net\\
            \(m_{\mathbf{w}}\) & the number of Monte Carlo samples drawn from task-specific posterior \(q(\mathbf{w}_{i}; \lambda_{i})\) \\
            \(m_{\theta}\) & the number of Monte Carlo samples drawn from hyper-meta posterior \(q(\theta; \psi)\) \\
            \(m_{0}^{(t)}\) & the number of samples in the training subset of each task\\
            \(m_{0}^{(v)}\) & the number of samples in the validation subset of each task\\
            \(m_{\phi}\) & the number of samples generated to train the \(\phi\) network in SImPa\\
            \(T\) & the number of tasks within a mini-batch to update the meta-parameter of interest\\
            \bottomrule
        \end{tabular}
        \caption{Notations used in the running time complexity analysis.}
        \label{tab:notations}
    \end{table}

    \subsection{Deterministic point estimate meta-learning (MAML)}
        \subsubsection{Lower-level optimisation for each task}
            The back-propagation of a single gradient update is \(\mathcal{O}(m_{0}^{(t)} n)\). This is then repeated \(N_{\mathrm{lower}}\) times. Hence, the total complexity to adapt to each task is \(\mathcal{O}(N_{\mathrm{lower}} m_{0}^{(t)} n)\).
        \subsubsection{Upper-level optimisation}
            Given the task-specific parameter \(\mathbf{w}_{i}\) obtained in the lower-level of \eqref{eq:meta_learning_objective}, one can calculate the gradient of the validation loss on each task w.r.t. the meta-parameter similarly to back-propagation. The complexity is, therefore, \(\mathcal{O}(m_{0}^{(v)} n)\). And since there are \(T\) tasks in total, the complexity of this step is \(\mathcal{O}(m_{0}^{(v)} T n)\).

        The total complexity of the whole algorithm is: \(\mathcal{O}((N_{\mathrm{lower}} m_{0}^{(t)} + m_{0}^{(v)} T) n)\).

    \subsection{Probabilistic meta-learning with multivariate normal distributions}
        These methods have similar complexity as the deterministic point estimate method, except the association of Monte Carlo sampling to draw \(\mathbf{w}_{i}\) from \(q(\mathbf{w}_{i}; \lambda_{i})\) when evaluating the training and validation losses.
        
        \subsubsection{Lower-level optimisation for each task}
            The optimisation for the lower-level in \eqref{eq:meta_learning_objective} now has 2 steps:
            \begin{itemize}
                \item Sampling for task-specific parameter: \(\mathbf{w}_{i} \sim q(\mathbf{w}_{i}; \lambda_{i})\), resulting in a complexity of \(\mathcal{O}(n)\)
                \item Back-propagation with complexity of \(\mathcal{O}(m_{0}^{(t)} n)\).
            \end{itemize}
            These steps are repeated from \(m_{\mathbf{w}}\) samples, and iterated \(N_{\mathrm{lower}}\), resulting in a total complexity of \(\mathcal{O}(m_{\mathbf{w}} N_{\mathrm{lower}} (m_{0}^{(t)} + 1) n)\).
        \subsubsection{Upper-level optimisation}
            Similarly, the upper-level is also affected by the Monte Carlo sampling of task-specific parameter \(\mathbf{w}_{i}\) from \(q(\mathbf{w}_{i}; \lambda_{i})\). This results in a complexity of \(\mathcal{O} ( m_{\mathbf{w}} T (m_{0}^{(v)} + 1) n )\).

        The total running time complexity is, therefore, \(\mathcal{O}( m_{\mathbf{w}} ( N_{\mathrm{lower}} (m_{0}^{(t)} + 1) + T (m_{0}^{(v)} + 1) ) n )\).

    \subsection{SImPa}
        The forward pass of SImPa is more complicated than the deterministic point estimate and the probabilistic modelling mentioned above. The reason is that the parameter of interest is not the meta-parameter \(\theta\), but the hyper-parameter \(\psi\) or \(\bm{\mu}_{\theta}\), which is one level higher in the hierarchical structure.

        In addition, we assume that the generator is much larger than the base model; \(n^{\prime} \gg n\).


        \subsubsection{Lower-level optimisation}
            The optimisation for the lower-level is carried out in 2 main steps. The first step is to train the \(\phi\) network which consists of:
            \begin{itemize}
                \item Sample \(m_{\phi}\) latent noise vectors \(\{\mathbf{z}_{\kappa}\}_{\kappa = 1}^{m_{\phi}}\): \(\mathcal{O}(m_{\phi} z)\) with \(z\) being assumed to be much less than \(n^{\prime}\)
                \item Forward pass to generate \(\mathbf{w}_{i}\) from the generator represented by the implicit distribution \(q(\mathbf{w}_{i}; \lambda_{i})\): \(\mathcal{O}(m_{\phi} n^{\prime})\)
                \item Sample \(\mathbf{w}_{i}\) from the prior \(p(\mathbf{w})\): \(\mathcal{O}(m_{\phi} n)\)
                \item Back-propagate to train the \(\phi\) network: \(\mathcal{O}(2 m_{\phi} n^{\prime})\)
            \end{itemize}
            This process is repeated \(N_{\phi}\) times, resulting in a total complexity of \(\mathcal{O}(3 N_{\phi} m_{\phi} n^{\prime} )\) to train \(\phi\)-network.

            Given the trained \(\phi\)-network for a particular task, the second step is to adapt the meta-parameter to the task-specific parameter:
            \begin{itemize}
                \item Generate \(\mathbf{w}_{i}\) from the generator: \(\mathcal{O}(n^{\prime})\)
                \item Back-propagation to train the task-specific generator: \(\mathcal{O}(m_{0}^{(t)} n^{\prime})\).
            \end{itemize}
            This is repeated \(m_{\mathbf{w}}\) times, resulting in a complexity of \(\mathcal{O}(m_{\mathbf{w}} (m_{0}^{(t)} + 1) n^{\prime}\).

            And again, this whole process of training both the \(\phi\)-network and the generator is repeated \(N_{\mathrm{lower}}\) times. Thus, the total complexity is \(\mathcal{O}(N_{\mathrm{lower}} (m_{\mathbf{w}} (m_{0}^{(t)} + 1) + 3 N_{\phi} m_{\phi}) n^{\prime})\).

        \subsubsection{Upper-level optimisation}
            Given the task-specific generator obtained in the optimisation of the lower-level in \eqref{eq:meta_learning_objective}, the gradient of the hyper-meta-parameter in the upper-level can be calculated in 2 steps:
            \begin{itemize}
                \item Sample \(\theta\) from \(q(\theta; \psi)\): \(\mathcal{O}(n^{\prime})\)
                \item Backward to calculate the gradient: \(\mathcal{O}(m_{0}^{(v)} n^{\prime})\)
            \end{itemize}
            This is done for \(T\) tasks, resulting in a complexity of \(\mathcal{O}( T (m_{0}^{(v)} + 1) n^{\prime})\) per a single \(\theta\).

        The optimisation in both the lower- and upper-levels is then repeated \(m_{\theta}\) times for \(m_{\theta}\) samples of \(\theta\). Thus, the total complexity of SImPa is: \(\mathcal{O}(m_{\theta} (N_{\mathrm{lower}} (m_{\mathbf{w}} (m_{0}^{(t)} + 1) + 3 N_{\phi} m_{\phi}) + T (m_{0}^{(v)} + 1)) n^{\prime})\).

        To ease the analysis, we specify the running time complexity of the three methods in \tableautorefname~\ref{tab:complexity_full}.

        \begin{table}[hb]
            \centering
            \begin{tabular}{l l l}
                \toprule
                \multirow{2}{*}{\bfseries Method} & \multicolumn{2}{c}{\bfseries Complexity}\\
                \cmidrule{2-3}
                 & General & In practice\tablefootnote{assume the sampling of \(N_{\phi}\) samples is parallel in GPU}\\
                 \midrule
                Deterministic & \(\mathcal{O}((N_{\mathrm{lower}} m_{0}^{(t)} + m_{0}^{(v)} T) n)\) & \(\mathcal{O}(100 n)\)\\
                Probabilistic & \(\mathcal{O}( m_{\mathbf{w}} ( N_{\mathrm{lower}} (m_{0}^{(t)} + 1) + T (m_{0}^{(v)} + 1) ) n )\) & \(\mathcal{O}(440 n )\)\\
                SImPa & \(\mathcal{O}(m_{\theta} (N_{\mathrm{lower}} (m_{\mathbf{w}} (m_{0}^{(t)} + 1) + 3 N_{\phi} m_{\phi}) + T (m_{0}^{(v)} + 1)) n^{\prime})\) & \(\mathcal{O}(125 n^{\prime})\)\\
                \bottomrule
            \end{tabular}
            \caption{Running time complexity per one gradient update of different meta-learning methods.}
            \label{tab:complexity_full}
        \end{table}

        To make it even easier to compare, we also add a \say{practical} setting into \tableautorefname~\ref{tab:complexity_full}. This setting is the one done in our experiments with the following hyper-parameter values:
        \begin{itemize}
            \item \(N_{\mathrm{lower}} = 5\)
            \item \(N_{\phi} = 1\)
            \item \(m_{0}^{(t)} = 5\) and \(m_{0}^{(v)} = 15\)
            \item \(m_{\theta} = m_{\mathbf{w}} = 1\) for SImPa and \(m_{\mathbf{w}} = 4\) for probabilistic methods that are based on multivariate normal distributions with diagonal covariance matrices
            \item \(T = 5\)
            \item \(N_{\phi} = 256\).
        \end{itemize}

        We note that despite the difference of the running time complexity, in the implementation using GPU, some operations, such as matrix multiplication, are implemented efficiently in parallel or vectorisation. Hence, the difference of running time in practice might not be the same as the one analysed in this appendix.

\end{document}